%% file: neurips_main.tex
\pdfoutput=1
\documentclass{article}

    \usepackage[preprint]{neurips_2022}

\usepackage[utf8]{inputenc} 
\usepackage[T1]{fontenc}    
\usepackage{hyperref}       
\usepackage{url}            
\usepackage{booktabs}       
\usepackage{amsfonts}       
\usepackage{nicefrac}       
\usepackage{microtype}      
\usepackage{xcolor}         

\usepackage{amsmath}
\usepackage{amssymb}
\usepackage{mathtools}
\usepackage{amsthm}
\usepackage{bm}
\usepackage{wrapfig}

\usepackage[capitalize,noabbrev]{cleveref}

\theoremstyle{plain}
\newtheorem{theorem}{Theorem}[section]
\newtheorem{proposition}[theorem]{Proposition}

\theoremstyle{definition}

\theoremstyle{remark}

\input{math_commands}

\usepackage[textsize=tiny]{todonotes}

\newcommand{\pa}[1]{\left(#1\right)}
\newcommand{\br}[1]{\left[#1\right]}

\renewcommand{\ro}{\partial}
\renewcommand{\eps}{\varepsilon}
\newcommand{\ba}[1]{\overline{#1}}
\newcommand{\del}{\nabla}
\renewcommand{\L}{\mathcal{L}}

\newcommand{\V}[1]{{\mathbf{#1}}} 

\newcommand*\ryu[1]{\textcolor{magenta}{[RY: #1]}}
\newcommand{\jl}[1]{{\color{blue}[JL: #1]}}
\newcommand{\nd}[1]{{\color{orange}[ND: #1]}}
\newcommand{\ry}[1]{{\color{magenta}[RY: #1]}}

\renewcommand{\ry}[1]{}
\renewcommand{\nd}[1]{}
\renewcommand{\jl}[1]{}
\renewcommand{\ryu}[1]{}

\newcommand{\out}[1]{}

\title{Faster  Optimization on Sparse Graphs \\ via Neural Reparametrization}

\author{
  Nima Dehmamy\thanks{Equal contribution} \\
  IBM Research\\
  \texttt{nima.dehmamy@ibm.com}
  \And
  Csaba Both$^*$\\
  Northeastern University\\
  \texttt{both.c@northeastern.edu}
  \And
  Jianzhi Long\\
  UCSD\\
  \texttt{jlong@ucsd.edu}
  \And
  Rose Yu\\
  UCSD\\
  \texttt{roseyu@ucsd.edu}
}

\begin{document}

\maketitle

\begin{abstract}
    In mathematical optimization, second-order Newton's methods generally converge faster than first-order methods, but they   require the inverse of the Hessian, hence  are computationally expensive.
    However, we discover that on sparse graphs, graph neural networks (GNN) can implement an efficient Quasi-Newton method that can speed up optimization by a factor of 10-100x.
    Our method, \textit{neural reparametrization},  modifies the optimization parameters as the output of a GNN to reshape the optimization landscape. Using a precomputed Hessian as the propagation rule, the GNN can effectively utilize the second-order information, reaching a similar effect as adaptive gradient methods.
    As our method solves optimization through  architecture design, it can be used in conjunction with any optimizers such as Adam and RMSProp.
    We show the application of our method on scientifically relevant problems including heat diffusion, synchronization and persistent homology.
    \out{
    We consider a class of optimization problems that involve dynamic processes on graphs.
    We discover that reparametrizing the optimization variables as the output of a neural network can lead to significant speedup.
    We examine the dynamics of gradient flow of such neural reparametrization.
    We find that to obtain the maximum speed up,  the neural network architecture needs to be a specially designed graph convolutional network (GCN).
    The aggregation function of the GCN is constructed from the gradients of the loss function and reduces to the Hessian in early stages of the optimization.
    We show the utility of our method on two optimization problems: network synchronization and persistent homology optimization, and find an impressive speedup, with our method being $4\sim 80 \times$  faster.
    }
\end{abstract}

\input{secs/introduction}

\section{Related Work}
\input{secs/related}

\input{secs/theory}



\section{Experiments}
\input{secs/experiment}

\section{Conclusion}
\input{secs/conclusion}





\bibliographystyle{icml2022}
\bibliography{references}

\section*{Checklist}

The checklist follows the references.  Please
read the checklist guidelines carefully for information on how to answer these
questions.  For each question, change the default \answerTODO{} to \answerYes{},
\answerNo{}, or \answerNA{}.  You are strongly encouraged to include a {\bf
justification to your answer}, either by referencing the appropriate section of
your paper or providing a brief inline description.  For example:
\begin{itemize}
  \item Did you include the license to the code and datasets?
  \item Did you include the license to the code and datasets? \answerNo{The code and the data are proprietary.}
  \item Did you include the license to the code and datasets? \answerNA{}
\end{itemize}
Please do not modify the questions and only use the provided macros for your
answers.  Note that the Checklist section does not count towards the page
limit.  In your paper, please delete this instructions block and only keep the
Checklist section heading above along with the questions/answers below.

\begin{enumerate}

\item For all authors...
\begin{enumerate}
  \item Do the main claims made in the abstract and introduction accurately reflect the paper's contributions and scope?
    \answerYes{}
  \item Did you describe the limitations of your work?
    \answerYes{}
  \item Did you discuss any potential negative societal impacts of your work?
    \answerNA{}
  \item Have you read the ethics review guidelines and ensured that your paper conforms to them?
    \answerYes{}
\end{enumerate}

\item If you are including theoretical results...
\begin{enumerate}
  \item Did you state the full set of assumptions of all theoretical results?
    \answerYes{}
        \item Did you include complete proofs of all theoretical results?
    \answerYes{}
\end{enumerate}

\item If you ran experiments...
\begin{enumerate}
  \item Did you include the code, data, and instructions needed to reproduce the main experimental results (either in the supplemental material or as a URL)?
    \answerYes{}
  \item Did you specify all the training details (e.g., data splits, hyperparameters, how they were chosen)?
    \answerYes{}
        \item Did you report error bars (e.g., with respect to the random seed after running experiments multiple times)?
    \answerYes{}
        \item Did you include the total amount of compute and the type of resources used (e.g., type of GPUs, internal cluster, or cloud provider)?
    \answerYes{}
\end{enumerate}

\item If you are using existing assets (e.g., code, data, models) or curating/releasing new assets...
\begin{enumerate}
  \item If your work uses existing assets, did you cite the creators?
    \answerYes{}
  \item Did you mention the license of the assets?
    \answerNA{}
  \item Did you include any new assets either in the supplemental material or as a URL?
    \answerNA{}
  \item Did you discuss whether and how consent was obtained from people whose data you're using/curating?
    \answerTODO{}
  \item Did you discuss whether the data you are using/curating contains personally identifiable information or offensive content?
    \answerNA{}
\end{enumerate}

\item If you used crowdsourcing or conducted research with human subjects...
\begin{enumerate}
  \item Did you include the full text of instructions given to participants and screenshots, if applicable?
    \answerNA{}
  \item Did you describe any potential participant risks, with links to Institutional Review Board (IRB) approvals, if applicable?
    \answerNA{}
  \item Did you include the estimated hourly wage paid to participants and the total amount spent on participant compensation?
    \answerNA{}
\end{enumerate}

\end{enumerate}


\newpage
\appendix

\input{secs/appendix}

\end{document}

%% file: math_commands.tex
\def\eqref#1{equation~\ref{#1}}
\def\Eqref#1{Equation~\ref{#1}}

\def\1{\bm{1}}

\def\eps{{\epsilon}}

\def\ro{{\textnormal{o}}}

\def\vv{{\bm{v}}}
\def\vw{{\bm{w}}}

\def\mH{{\bm{H}}}

\def\mW{{\bm{W}}}

\def\mZ{{\bm{Z}}}

\DeclareMathAlphabet{\mathsfit}{\encodingdefault}{\sfdefault}{m}{sl}
\SetMathAlphabet{\mathsfit}{bold}{\encodingdefault}{\sfdefault}{bx}{n}

\newcommand{\E}{\mathbb{E}}

\newcommand{\R}{\mathbb{R}}

\DeclareMathOperator{\Tr}{Tr}

%% file: secs/introduction.tex
%

\begin{figure}[h]
 \vskip  -0.2in
 \centering
 \includegraphics[width=1\linewidth]{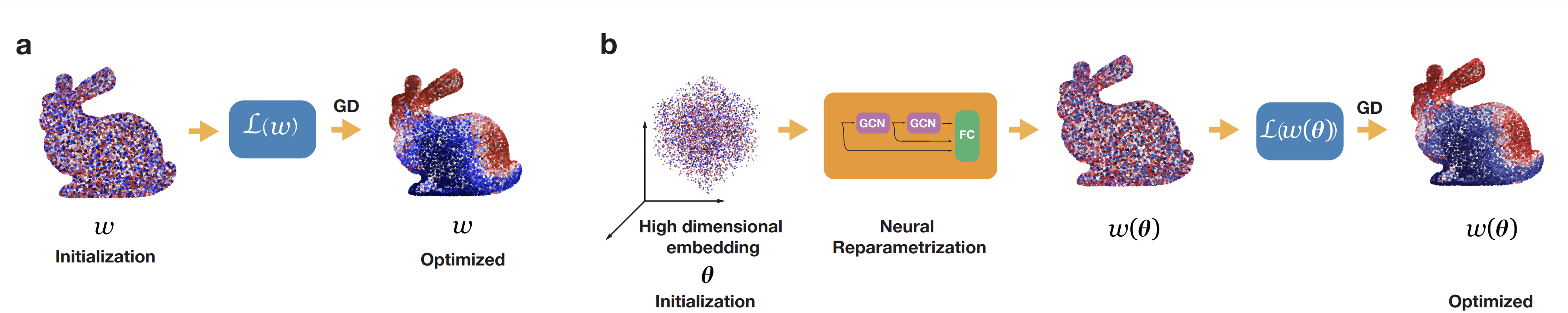}
 \caption{Original (a) vs Neural Reparametrization using GNN (b). 
 In problems on sparse graphs, using the Hessian inside the GNN implements a quasi-Newton method and accelerates optimization.
 }
 \label{fig:overview}
\end{figure}

Dynamical processes on graphs are ubiquitous in many real-world problems
such as traffic flow on road networks, epidemic spreading on mobility networks, and heat diffusion on a surface.
They also frequently appear in machine learning applications including  accelerating fluid simulations \cite{ummenhofer2019lagrangian, pfaff2020learning}, topological data analysis \cite{carriere2021optimizing, birdal2021intrinsic}, and solving differential equations \cite{zobeiry2021physics, he2020unsupervised, schnell2021half}.
Many such problems can be cast as (generally highly non-convex) optimization problems
on graphs.
Gradient-based methods are often used for numerical optimization.
But on large graphs, they also suffer from slow convergence \cite{chen2018fastgcn}.

In this paper,  we propose a novel optimization method based on  \textit{neural reparametrization}: parametrizing the solution to the optimization problem with a graph neural network (GNN).  Instead of optimizing the original parameters of the problem, we optimize the weights of the graph neural network.  Especially  on large sparse graphs, GNN reparametrization can be implemented efficiently, leading to significant speed-up.

The effect of this reparametrization is similar to adaptive gradient methods. As shown in \cite{duchi2011adaptive}, when using gradient descent (GD) for loss $\L$ with parameters $\vw$, optimal convergence rate is achieved when the learning rate is  proportional to $ G^{-1/2}_t := ( \sum_{\tau=1}^t {g}_\tau {g}^\top_\tau )^{-1/2}$ constructed from the gradient of the loss ${g}_t = \nabla \L (\vw_t)$ in the past $t$ steps.
%
However, when $G_t$ is large, computing $G_t^{-1/2}$ during optimization is intractable. Many adaptive gradient methods such as  AdaGrad \cite{duchi2011adaptive}, RMSProp \citep{rmsprop} and Adam \citep{kingma2014adam}  use a diagonal approximation of $G_t$.
More recent methods such as KFAC \citep{martens2015optimizing-kfac} and Shampoo \citep{gupta2018shampoo,anil2020scalable}
approximate a more compressed version of $G$ using Kronecker factorization for faster optimization.

For sparse graph optimization problems, we can construct a GNN to approximate $G_t^{-1/2}$ efficiently.
This is achieved by using a precomputed Hessian as the propagation rule in the GNN.
Because the GNN is trainable, its output also changes \textit{dynamically} during GD.  Our method effectively utilizes the second-order information, hence it is also implicitly quasi-Newton. But rather than approximating the inverse Hessian, we change the optimization variables entirely.
In summary, we show that
\begin{enumerate}
    \item Neural reparametrization, an optimization technique that  reparametrizes the optimization variables with  neural networks  has a similar effect to adaptive gradient methods.
    \item A particular linear reparametrization using GNN recovers the optimal adaptive learning rate of AdaGrad \citep{duchi2011adaptive}.
    \item On sparse graph optimization problems in early steps, a GNN reparametrization is computationally efficient, leading to faster convergence (400\%-8,000\% speedup). 
    \item We show the effectiveness of this method on three scientifically relevant problems on graphs: heat diffusion, synchronization, and persistent homology.
\end{enumerate}


%% file: secs/related.tex
\paragraph{Graph Neural Networks.} Our method demonstrates a novel and unique perspective for GNN. 
The majority of  literature  use GNNs to learn representations from graph  data to make predictions,  see surveys and the references in \citet{bronstein2017geometric, zhang2018deep, wu2019comprehensive,goyal2018graph}. 
Recently, \citep{bapst2020unveiling} showed the power of GNN in predicting long-time behavior of glassy systems, which are notoriously slow and difficult to simulate. 
Additionally, \citep{fu2022simulate} showed that GNN-based models can help speedup simulation of molecular dynamics problems by predicting large time steps ahead. 
However, we use GNNs to modify the learning  dynamics of  sparse graph optimization problems.   We discover that by reparametrizing optimization problems, we can have significantly speed-up. 
Indeed, we show analytically that a GNN with a certain aggregation rule achieves the same optimal adaptive learning rate as in \citep{duchi2011adaptive}. Thanks to the sparsity of the graph, we can obtain an efficient implementation of the optimizer that mimics the behavior of quasi-Newton methods.

\paragraph{Neural Reparametrization.}
Reparameterizing an optimization problem can reshape the landscape geometry, change the learning dynamics,  hence speeding-up convergence. In linear systems, preconditioning \citet{axelsson1996iterative, saad2000iterative} reparameterizes the problem by multiplying a fixed symmetric positive-definite preconditioner matrix to the original problem. \cite{groeneveld1994reparameterization} reparameterizes the covariance matrix to allow the use a faster Quasi-newton algorithm in maximum likelihood estimation. 
Recently, an \textit{implicit acceleration} has been documented in over-parametrized linear neural networks and analyzed in \citet{arora2018optimization, tarmoun2021understanding}. 
Specifically, \citet{arora2018optimization} shows that reparametrizing with deep linear networks impose a preconditioning scheme on gradient descent. 
Other works \citep{sosnovik2019neural, hoyer2019neural}
have demonstrated that reparametrizing with convolutional neural networks can speed-up structure optimization problems (e.g. designing a bridge).
But the theoretical foundation for the improvement is not well understood. To the best of our knowledge, designing GNNs to reparametrize and accelerate graph optimization has not been studied before.
%


\paragraph{PDE Solving.} Our work can also be viewed as finding the steady-state solution of non-linear discretized partial differential equations (PDEs). Traditional finite difference or finite element methods \cite{} for solving PDEs are computationally challenging. Several recent works use deep learning to solve PDEs in a \textit{supervised} fashion. For example, physics-informed neural network (PINN) \cite{raissi2019physics, greenfeld2019learning, bar2019unsupervised} parameterize the solution of a PDE with neural networks. Their neural networks take as input the physical domain. Thus, their solution is independent of the  mesh but is specific to each parameterization. Neural operators \cite{lu2021learning, li2020neural, li2020fourier} alleviate this limitation by learning in the space of  infinite dimensional functions. However, both class of methods require data from the numerical solver as supervision, whereas our method is completely \textit{unsupervised}. We solve the PDEs by directly minimizing the energy function.

%% file: secs/theory.tex
\out{
\section{Optimal Preconditioning}
The key result for adaptive gradients, due to \cite{duchi2011adaptive}, is that optimal convergence rate is achieved when the learning rate is $\eps\sim \eps_0 G^{-1/2}$ where $G \sim \sum_t \nabla L \nabla \L^T $ is constructed from gradient of the loss $\L$ in past steps $t$. 
}

\section{Gradient Flow Dynamics and Adaptive Gradient}
Consider a graph with $n$ nodes (vertices) and each node $i$ has a  state vector $\vw_i \in \R^d$. We have an adjacency matrix $A \in \R^{n\times n}$ where $A_{ij}$ is the weight of the edge from node $i$ to node $j$. 
We look for a matrix of states 
$\vw \in \R^{n\times d}$ that minimizes a loss (energy) function $\L(\vw)$.

\paragraph{Gradient Flow.}
The optimization problem above can be tackled using 
gradient-based methods. Although in non-convex settings, these methods likely won't find a global minimum. 
Gradient descent (GD) updates the variables $\vw$ by taking repeated steps in the direction of the steepest descent $\vw_{t+1} = \vw_t - \eps \frac{\ro \L}{\ro\vw_t}$ for a learning rate $\eps$.
With infinitesimal time steps $\delta t$, GD becomes the continuous time gradient flow (GF) dynamics: 
\begin{align}
    {d\vw\over dt}& = -\eps {\ro \L \over \ro \vw },
    \label{eq:GF}
\end{align}
\paragraph{Adaptive Gradient.}
In general, $\eps$ can be a (positive semi-definite) matrix $\eps \in \R^{n\times n}$ that vary per parameter and change dynamically at each time step. 
In some parameter directions, GF is much slower than others \cite{mcmahan2010adaptive}.
Adaptive gradient methods use different learning rates for each parameter to make GF isotropic (i.e. all directions flowing at similar rates). The adaptive learning rate is: 
\begin{align}
    \eps & = \eta G^{-1/2} \in \mathbb{R}^{n\times n}, & 
    G \equiv \E\br{\del \L \del \L^T },
    \label{eq:eps-adagrad}
\end{align}
where $\eta \ll 1$ is a small constant.
The expectation $\E$ can be defined either  over a mini-batch of samples or over multiple time steps.  
In AdaGrad \citep{duchi2011adaptive}, $\E$ is  over some past time steps, while in RMSProp \citep{rmsprop} and Adam \citep{kingma2014adam} it is a discounted time averaging. 
Defining the gradient $g(t) = \del \L(\vw(t)) $, this expectation can be written as 
\begin{align}
    \E_{\Delta t}\br{f}(t)& \equiv  {1\over \Delta t}  \int_0^{\Delta t} ds 
    \gamma^s f(t-s), &  
    G(t) &\equiv \E_{\Delta t}\br{ gg^T }(t)  
    \label{eq:Gt-adagrad}
\end{align}
where $\gamma < 1$ is the discount factor.
Unfortunately, $G(t) \in \R^{n\times n}$ is generally a large matrix and computing $\eps = \eta G^{-1/2}$ ($O(n^3)$) during optimization is not feasible. 
Even using a fixed precomputed $G^{-1/2}$ is expensive, being $O(n^2)$ for the matrix times vector multiplication $G^{1/2} \del \L$. 
Therefore, methods like AdaGrad, Adam and RMSprop use a diagonal approximation $\mathrm{diag}(g^2_\tau)$ in \eqref{eq:Gt-adagrad}, while Shampoo and K-FAC use a more detailed Kronecker factorized approximation of $G_\tau$.


\section{Neural Reparametrization}
For  sparse graphs, a better approximation of $G^{-1/2}$ can be used in early steps to speed up  optimization.
The key idea behind our method is to  change the optimization parameters with a graph neural network (GNN) whose  propagation rule involves an approximate Hessian. 
We add a GNN module  such that GF on the new problem is equivalent to a quasi-Newton's method on the original problem. 
GNN allows us to perform quasi-Newton's method with time complexity similar to first-order GF, with $O(nk)$ complexity with $k \ll n$ proportional to average degree of nodes.

After a few  iterations, the approximate Hessian starts deviating significantly from the true Hessian and the reparametrization becomes less beneficial. 
Therefore, once improvements in the objective function become small,
we switch back to GF on the original problem. 
Using this two stage method we observe impressive speedups. Figure \ref{fig:overview} visualizes the pipeline of the proposed approach.

\out{
We first review how $G^{-1/2}$ and the inverse Hessian are identical in early optimization steps.
We then show that this inverse Hessian can be approximated with a reparametrization using GNN. 
}

\subsection{Proposed Approach}
We reparametrize the problem by expressing the optimization variable $\vw $ as a neural network function $\vw(\theta) $, where   $\theta$ are the trainable parameters. Rather than optimizing over $\vw$ directly, we optimize over the neural network parameters $\theta$. 
We seek a 
neural network architecture that is guaranteed to accelerate  optimization after reparametrization.
If we want to match the adaptive learning rate in \eqref{eq:eps-adagrad}, this would naturally lead to the design of   $\vw(\theta) $ as a GNN. 
We begin by comparing the GF and rate of loss  decay $d\L/dt$ for $\vw(\theta)$ with the original ones for $\vw$.

\paragraph{Modified Gradient Flow.}
After reparametrization $\vw(\theta)$, we are updating $\theta $ using GF on $\L(\vw(\theta))$
\begin{align}
    {d\theta_a \over dt} = -\sum_b \hat{\eps}_{ab} {\ro \L \over \ro \theta_b} = -\sum_b\hat{\eps}_{ab} \sum_i  {\ro \L \over \ro \vw_i } {\ro \vw_i \over \ro \theta_b} = -
    \br{\hat{\eps}J \del \L }_a
    \label{eq:d-theta-dt}
\end{align}
where $\hat{\eps}$ is the learning rate for parameters $\theta$, and $J \equiv \ro \vw / \ro \theta $ is the Jacobian of the reparametrization.  
Note that $\del \L = \ro \L /\ro \vw $. 
From \eqref{eq:d-theta-dt} we can also calculate the $d\vw (\theta)/dt$
\begin{align}
    {d\vw \over dt} & = 
    {\ro \vw \over \ro \theta }^T 
    {d\theta \over dt} = -
    J^T\hat{\eps}J \del \L 
    \label{eq:dwdt-theta-reparam}
\end{align}
which means that $d\vw/dt $ has now acquired an adaptive learning rate $J^T\hat{\eps}J$. Therefore,
the choice of architecture for $\vw(\theta)$ would determine $J$ and hence the convergence rate.

\paragraph{Architecture Choice.} 
We can show that an adaptive, linear reparametrization can closely mimic the optimal adaptive learning rate of \citep{duchi2011adaptive}.
This motivates the GNN architecture we use below for graph optimization problems. 
\begin{proposition}
For $\theta \in \R^m$ and $\vw \in \R^n$, with $m\geq n$, using a linear reparametrization $\vw = J \theta $ leads to the optimal adaptive learning rate in \eqref{eq:eps-adagrad},  where ($\gamma \in \R^{n \times m}$), 
\begin{align}
    J & = \sqrt{\eta} \gamma G_t^{-1/4}, & 
    \gamma^T\hat{\eps} \gamma &= I_{n\times n}
    \label{eq:JepsG-sol}
\end{align}
\end{proposition}
\begin{proof}
As before, the learning rate $\hat{\eps}$ must be PSD. 
Thus, for $m \geq n$ and using SVD we can find $\gamma $ such that $\gamma^T \hat{\eps} \gamma = I_{n\times n}$. 
It follows that $J$ in \eqref{eq:JepsG-sol} satisfies $J^T\hat{\eps}J  = \eta G_t^{-1/2} $. 
\end{proof}
The solution \eqref{eq:JepsG-sol} is not unique. 
Even for $m=n $, $\gamma$ and $G_t^{-1/4}$ are not unique and any such solution yields a valid $J$ (e.g. any $vv^T=G_t^{-1/2}$ with large hidden dimension works). 
\out{
If we use simple SGD for the reparametrized model $\vw(\theta)$ with a constant learning rate $\hat{\eps} = \eta I $ we get $ J^T J  = G_t^{-1/2}$. 
\out{
\begin{align}
    J^T J & = G_t^{-1/2}.
    \label{eq:JTJ-G}
\end{align}
}
This can be satisfied in different ways.
One way is to have a linear architecture where $\vw = J \theta $. 
For $\theta \in \R^m$ and $\vw \in \R^n$, we need $m\geq n$ to satisfy \eqref{eq:JTepsJ-G} exactly. 
For $m=n$, we can obtain $J$ using the spectral expansion of $G_t$, which is also  
\out{
Let $\psi_i$ denote the eigenvector of $G_t$ with eigenvalue $\lambda_i$ (note $G$ is positive semi-definite (PSD), i.e. $\lambda_i\geq 0$) and $\Psi= (\psi_1,\cdots, \psi_n)$ be the matrix of all eigenvectors.  
\begin{align}
    G_t&= \sum_i \lambda_i \psi_i\psi^T = \Psi \Lambda \Psi^T, & 
    J&= \Lambda^{-1/4} \Psi^T 
\end{align}
}
expensive ($O(n^3)$).
}%
However, even when $\hat{\eps}=\eta I $ is constant, obtaining $J$ requires expensive spectral expansion of $G_t$ of $O(n^3)$. 
We could use a low-rank  approximation of  $J$ by minimizing the error $ \| J^TJ - G_t^{-1/2}\|^2$ with $ \theta \in \R^m$.
As $G_t$ changes during optimization, 
we also need to update $J$. 
For a small number of iterations after $t$, the change is small. So a fixed  $J$ could still yield a good approximation $J\approx G_t^{-1/4}$. 
However, for sparse graphs, we show that  efficient approximations to $\vw \approx G_t^{-1/4} \theta $ can be achieved via a  GNN. 

\paragraph{Normalization} 
Note that in GF with adaptive gradients $d\vw/dt = -\eta G_t^{-1/2} \del \L$, the choice of the learning rate $\eta $ depends on the eigenvalues of $G_t$. 
To ensure numerical stability we need $\eta \ll 1/\sqrt{\lambda_{max}}$, where $\lambda_{max}$ is the largest eigenvalue of $G_t$. 
If we normalize $G_t \to G_t / \lambda_{max}$, we don't need to adjust $\eta$ and any $\eta \ll 1$ works. 
Therefore, we want the Jacobian to satisfy 
\begin{align}
    J& \approx \pa{{G_t\over \lambda_{max}}}^{-1/4} 
    \label{eq:J-G-norm}
\end{align}


\subsection{Efficient Implementation for Graph Problems \label{sec:implement} }
Note that $G_t^{-1/2}$ in adaptive gradient methods approximates the inverse Hessian \citep{duchi2011adaptive}. 
If $\vw$ are initialized as $\vw_i \sim \mathcal{N}(0,1/\sqrt{n})$, we have (see SI \ref{ap:G-estimate})
\begin{align}
    G_{ij}(t\to 0)& =\sum_{k,l} \E[\vw_k\vw_l] {\ro^2 \L \over \ro \vw_k\ro \vw_i } {\ro^2 \L \over \ro \vw_l\ro \vw_j } + O(n^{-2}) 
    = {1\over n}\left.\br{\mathcal{H}^2}_{ij}\right|_{\vw \to 0} + O(n^{-2}), 
    \label{eq:M-Hessian2-short}
\end{align}
where $\mathcal{H}_{ij}(\vw)\equiv \ro^2 \L(\vw)/\ro \vw_i \ro \vw_j$ is the Hessian of the loss at $t=0$.  
Therefore in early stages, instead of computing $J\sim (G_t/\lambda_{max})^{-1/4}$, we can implement a quasi-Newton methods with
$J= \mH^{-1/2}$ where $\mH\equiv (1-\xi)\mathcal{H}/h_{max}$. 
Here $h_{max}\approx \sqrt{\lambda_{\max}}$ is the top eigenvalue of the Hessian and
$\xi\ll 1$. 
This requires pre-computing the Hessian matrix, but the denseness of $\mathcal{H}^{-1/2}$ will slow down the optimization process. Additionally, 
we also want to account for the changes in $\mathcal{H}$. 
Luckily, in graph optimization problems, sparsity can offer a way to tackle both issues using GNN.
We next discuss the structure of these problems. 
\looseness=-1
\out{
The eigenvectors of $G_t$ with zero eigenvalue are modes that do not evolve during optimization.
As gradients are zero in these directions, GD can never find and explore those directions. 
Therefore, we may exclude them from the beginning and claim that $G_t$ is full rank. 
}

\paragraph{Structure of Graph Optimization Problems.} 
For a graph with $n$ nodes (vertices), each of which has a state vector $\vw_i$, and an adjacency matrix $A \in \R^{n\times n}$, the graph optimization problems w.r.t. the state matrix $\vw \in \R^{n\times d}$ have the following common structure. 
\footnote{Although here $\vw$ is not flattened,
each row of $\vw$ still follows the same GF \eqref{eq:GF} and the results extend trivially to this case. }
\begin{align}
    \L(\vw) &\approx   \sum_{ij}A_{ij}\|\vw_i-\vw_j\|^2 + O(\vw^3) = \Tr\br{\vw^T L \vw} + O(\vw^3), 
    \label{eq:L-graph}
\end{align}
where $L=D-A$ is the graph Laplacian, $D_{ij} = \sum_{k} A_{ik} \delta_{ij}$ the diagonal degree matrix and $\delta_{ij}$ is the Kronecker delta. 
\eqref{eq:L-graph} is satisfied by all the problems we consider in our experiments, which include diffusion processes and highly nonlinear synchronization problems. 
With \eqref{eq:L-graph}, the Hessian at the initialization often becomes $\mathcal{H} \sim 2 L$. 
Hence, when the graph is sparse, $\mathcal{H}$ is sparse too.
Still, $\mathcal{H}^{-1}$ can be a dense matrix, making Newton's method expensive.  
However, at early stages,  we can exploit the sparsity  of Hessian to approximate $J = \mH^{-1/2}$ 
for efficient  optimization.

\paragraph{Exploiting Sparsity.} 
If $\mH$ 
is sparse or low-rank, we may estimate $\mH^{-1}$ using a short Taylor series (e.g. up to $O(\mH^2)$), which also remains sparse. 
When the graph is undirected and the degree distribution is concentrated (i.e. not fat-tailed) $h_{max} \approx \Tr[D]/n$  (average degree) (Appendix \ref{ap:implement}) 
\begin{align}
    \mH 
    \approx I - D^{-1/2}AD^{-1/2} = I - A_s. 
\end{align}
where $A_s = D^{-1/2}AD^{-1/2}$ is the symmetric degree normalized adjacency matrix. 
To get an $O(qn^2)$ approximation for this $J = \mH^{-1/2}$ wwe can use $q$ terms in the binomial expansion  $\mH^{-1/2}  \approx I- {1\over 2} A_s -{3\over 4} A_s^2+\dots $, which for small $q$ is also sparse.  
Next we show how such an expansion can be implemented using GNN.


\paragraph{GCN Implementation.}
A graph convolutional network (GCN) layer takes an input $\theta \in \R^{n\times d}$ 
and returns $\sigma(f(A) \theta V)$, where $f(A)$ is the propagation (or aggregation) rule, $ V \in \R^{d\times h}$ are the weights,  $\sigma$ is the nonlinearity.
For the GCN proposed by \citet{kipf2016semi}, we have $f(A) = A_s$. 
A linear GCN layer with residual connections represents the polynomial $F(\theta) = \sum_k A_s^k \theta V_k $.  \citep{dehmamy2019understanding}. 
We can implement an approximation of $\vw(\theta) = \mH^{-1/2}\theta $ using GCN layers. 
For example, we can implement the $O(A_s^2)$ approximation  $J\approx \mH^{-1/2}$  using a two-layer GCN with pre-specified weights. 
However, to account for the changes in the Hessian during optimization, we make  the  GCN weights trainable parameters.  
We also use a nonlinear GCN $\vw(\theta)=GNN(A_s, \theta)$, where $GNN(\cdot)$ is a trainable function implemented using one or two layers of GCN (Fig. \ref{fig:overview}).


\out{
\paragraph{Optimization Parameters}
We consider problems each node $i$ has a  state vector $\vw_i \in \R^d$. 
We are looking for a matrix of states 
$\vw \in \R^{n\times d}$ which minimizes a loss (energy) function $\L(\vw)$.
Although this $\vw$ has an additional dimension compared to the flattened $\vw$ discussed earlier, each row of $\vw$ still follows the same GF \eqref{eq:GF} and the results extend trivially to this case.
} 

\paragraph{Two-stage optimization.}
In $\vw(\theta) = GNN(A_s,\theta)$,  both the input $\theta$ and  the weights of the GCN layers in $GNN(\cdot)$ are trainable. 
In spite of this, because $GNN(\cdot)$ uses the initial Hessian via $A_s$, it may not help when the Hessian has changed significantly. 
This, combined with the extra computation costs, led us to adopt a two-stage optimization.
In the initial stage, we use our GNN reparametrization with a precomputed Hessian to perform a quasi-Newton GD on $\vw(\theta)$. 
Once the rate of loss decay becomes small, we switch to GD over the original parameters $\vw$, initialized using the final value of the parameters from the first stage. 

\paragraph{Per-step Time Complexity.}
Let $\vw\in \R^{n\times d}$ with $d\ll n$ and let the average degree of each node be $k = \sum_{ij} A_{ij}/n $. 
For a sparse graph we have $k\ll n$. 
Assuming the leading term in the loss $\L(\vw)$ is as in \eqref{eq:L-graph},  the complexity of computing $\del \L(\vw)$ is at least $O(dn^2k)$, because of the matrix product $L\vw$.  
When we reparametrize to $\theta \in \R^{n\times h}$ with $h\ll n$, passign through each layer of GCN has complexity $O(hn^2k)$. 
The complexity of GD on the reparametrized model with $l$ GCN layers is $O((lh+d)n^2k)$. 
Thus, as long as $lh$ is not too big, the reparametrization slows down each iteration by a constant factor of $\sim 1+lh/d$. 

\out{
Then the complexity of computing $\L(\vw)$ is $O(dnk)$ where $k$ depends on the sparsity structure of $\L$. 
Hence we note that using \eqref{eq:K-M-ideal} to speedup the optimization may generally be useful if the the per step complexity of GD on $\L$ has a time complexity of least $O(qn^2)$.
(GD steps $\sim O(n^2)$). 
Ideally we want an approximation for
$\kappa $ which is $O(qn^2)$ with $q\sim O(1)$.
}

%% file: secs/experiment.tex
We showcase the acceleration of neural reparametrization on three graph optimization problems: heat diffusion on a graph; synchronization of oscillators; and persistent homology, a mathematical tool that  computes the topology features of data.
We use the Adam \citet{kingma2014adam} optimizer and compare different reparametrization models.
We implemented all models with Pytorch.
Figure \ref{fig:speedup_summary} summarizes the speedups (wall clock time to run original problem divided by time of the GNN model)
we observed in all our experiments.
We explain the three problems used in the experiments next.
\looseness=-1

\out{
Consider a graph with $n$ nodes (vertices) and an adjacency matrix $A \in \R^{n\times n}$ where $A_{ij}$ is the weight of the edge from node $i$ to node $j$.
We consider problems each node $i$ has a  state vector $\vw_i \in \R^d$.
We are looking for a matrix of states
$\vw \in \R^{n\times d}$ which minimizes a loss (energy) function $\L(\vw)$.
Although this $\vw$ has an additional dimension compared to the flattened $\vw$ discussed above, each row of $\vw$ still satisfies the same equations as before and the results extend trivially to this case as we discuss here.
In an optimization problem on a graph the graph structure $A$ must play a prominent role in $\L(\vw)$.
To clarify this point, we briefly describe the loss functions for two of the problems used in our experiments.
}

\begin{figure}
    \vskip -0.2in
    \centering
    \includegraphics[width=.7\linewidth]{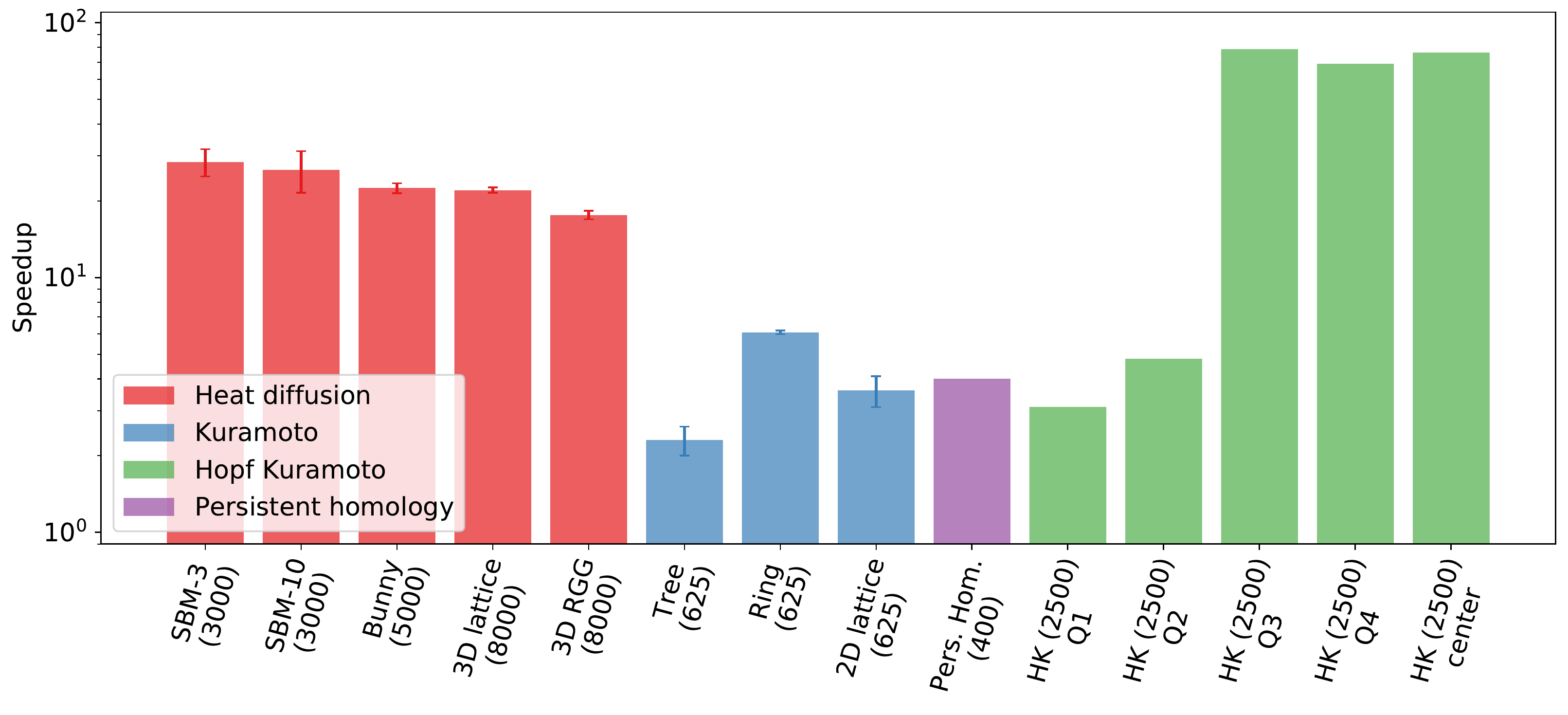}
    \vskip -0.18in
    \caption{Summary of the speedups in different experiments.
    The colors indicate the type of loss function and the labels explain the graph structure.
    Numbers in parentheses are the number of vertices.
    \looseness=-1
    }
    \vskip -0.2in
    \label{fig:speedup_summary}
\end{figure}

\begin{figure}[b]
 \vskip -0.2in
\centering
    \includegraphics[width=.9\linewidth]{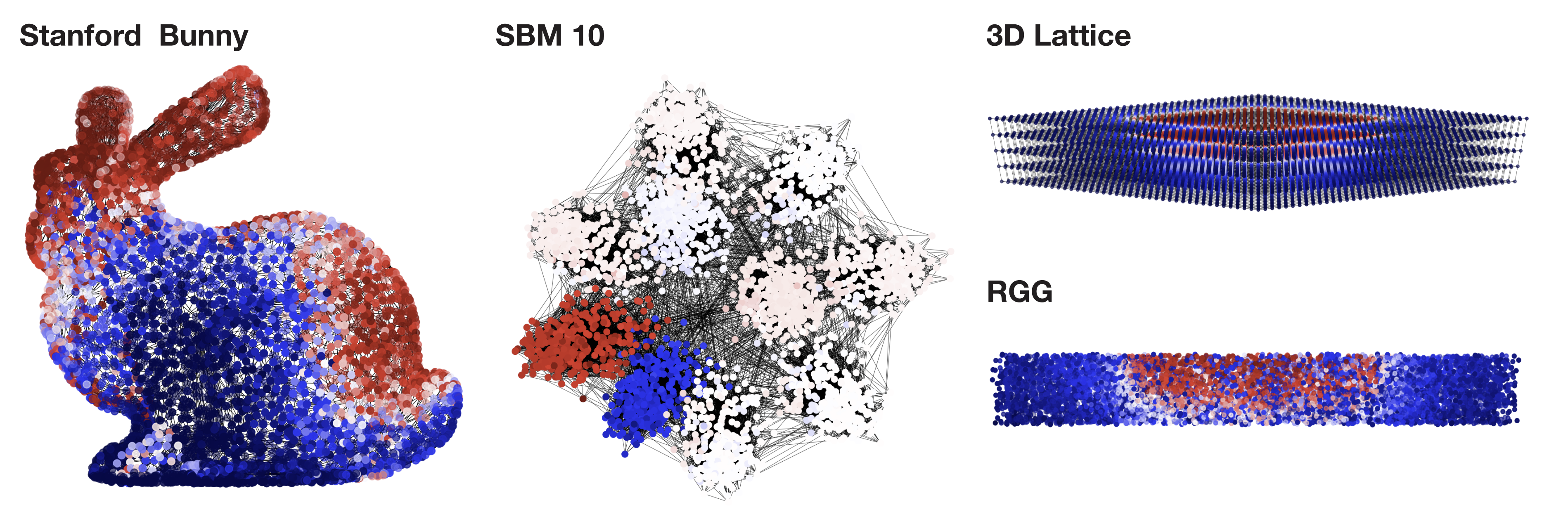}

 \caption{Optimized heat distribution on the Stanford bunny, Stochastic Block Model with 10 blocks (SBM 10), 3D lattice and Random Geometric Graphs (RGG). }
     \label{fig:heat_visualization}
\vskip -0.2in

\end{figure}

\subsection{Heat Diffusion}
Heat equation describes heat diffusion  \cite{incropera1996fundamentals}.
\nd{cite recent ML papers on heat eq}
It is given by $\ro_t \vw = - \eps \del^2 \vw $ where $ \vw(x,t) $ represents the temperature at point $x$ at time $t$, and $\eps$ is the heat diffusion constant.
On a graph, $\vw(x,t) \in \R^{+}$ is discretized and replaced by $\vw_i(t)$, with node $i$ representing the position.
The Laplacian operator $\del^2$  becomes the graph Laplacian $L = D-A$ where $D$ is the diagonal degree matrix with entries $D_{ij} = \sum_k A_{ik}\delta_{ij} $ (Kronecker delta).
The heat diffusion on graphs $d \vw/dt = -\eps L \vw $ 
can be derived as minimizing the following loss function
\begin{align}
    \L_{HE}(\vw)& = {1\over 2}\vw^T L \vw
    = {1\over 2 } \sum_{ij} A_{ij}(\vw_i-\vw_j)^2
\end{align}

While this loss function is quadratic and the heat equation is linear, the boundary conditions make it highly nonlinear.
For example, a set $S$ of the nodes may be attached to a heat or cold source with fixed temperatures $T_i$ for $i\in S$. In this case, we will
 add a regularizer $c\sum_{i\in S}\|\vw_i - T_i\|^4$ to the loss function.
For large meshes, lattices or amorphous, glassy systems, or systems with bottlenecks (e.g. graph with multiple clusters with bottlenecks between them),  finding the steady-state solution $\vw(t\to \infty)$ of heat diffusion can become prohibitively slow.

\textbf{Results for heat diffusion.}
Figure \ref{fig:speedup_summary} summarizes the observed speedups.
We find that on all these graphs, our method can speed up finding the final $\vw$ by over an order of magnitude.
Figure \ref{fig:heat_visualization} shows the final temperature distribution in some examples of our experiments.
We ran tests on different graphs, described next.
In all case we pick 10\% of nodes and connect them to a hot source using the regularizer $\|\vw_i-T_h\|^2$, and 10\% to the cold source using $\|\vw_i-T_c\|^2$.
The graphs in our experiments include the Stanford Bunny, Stochastic Block Model (SBM), 2D and 3D lattices, and Random Geometric Graphs (RGG) \cite{penrose2003random, karrer2011stochastic}.
SBM is model where the probability of $A_{ij} =1$ is drawn from a block diagonal matrix.
It represents a graphs with multiple clusters (diagonal block in $A_{ij}$) where nodes within a cluster are more likely to be connected to each other than to other clusters (Fig. \ref{fig:heat_visualization}, SBM 10).
RGG are graphs where the nodes are distributed in space (2D or higher) and nodes are more likely to connect to nearby nodes.

\subsection{Synchronization}
Small perturbations to many physical systems at equilibrium can be described a set of oscillators coupled over a graph
(e.g. nodes can be segments of a rope bridge and edges the ropes connecting neighboring segments.)
An important 
model for studying 
is the Kuramoto model \citep{kuramoto1975self,kuramoto1984chemical}, which has had a profound impact on engineering, physics, machine learning \cite{schnell2021half} and network synchronization problems \citep{pikovsky2003synchronization} in social systems.
The loss function for the Kuramoto model is defined as
\begin{align}
    \L(\vw) &= -\sum_{i,j} A_{ji} \cos\Delta_{ij}, &
    \Delta_{ij} &= \vw_i-\vw_j.
    \label{eq:Kuramoto}
\end{align}
which can be derived from the misalignement $ \|x_i-x_j\|^2 =2\br{1+\cos(\vw_i-\vw_j)} $ between unit 2D vectors $x_i$ representing each oscillator.
Its GF equation $d\vw_i /dt = -\eps \sum_{j} A_{ij} \sin\Delta_{ij}$, is highly nonlinear.
We further consider a more complex version of the Kuramoto model
important in physics:
the Hopf-Kuramoto (HK) model  \citet{lauter2015pattern}. 
\out{
The dynamics of the HK model follows
\begin{align}
    \frac{d\vw_i}{dt} &= c \sum_jA_{ji}\br{\cos\Delta_{ij} - s_1\sin\Delta_{ij}} \cr
    &+ s_2\sum_{k,j}  A_{ij} A_{jk} \br{\sin\pa{\Delta_{ji}+\Delta_{jk}} - \sin\pa{\Delta_{ji}-\Delta_{jk}}}  \cr
    &+A_{ij} A_{ik} \sin\pa{\Delta_{ji}+\Delta_{ki}}.
    \label{eq:Hopf-Kuramoto}
\end{align}
where $c$, $s_1,s_2$ are the governing parameters.
}
The loss function for the HK model is
\begin{align}
    \L &= 
    \sum_{i,j} A_{ji}\br{ \sin\Delta_{ij} + s_1\cos\Delta_{ij}}
    \!+\! {s_2\over 2}
    \sum_{i,k,j} A_{ij} A_{jk}  \big[\cos\pa{\Delta_{ji}\!+\!\Delta_{jk}}
    +\cos\pa{\Delta_{ji}\!-\!\Delta_{jk}}\big] \label{eq:Loss-Hopf-Kuramoto}
\end{align}
where $s_1,s_2$ are model constants determining the phases of the system.
This model has very rich set of phases (Fig. \ref{fig:hopf_kuramoto})
and
the phase space includes regions where simulations becomes slow and difficult.
This diversity of phases allows us to showcase our method's performance in different parameter regimes and in highly nonlinear scenarios.

\begin{figure}[b]
\vskip -0.2in
\begin{center}
\centerline{
    \includegraphics[width=\linewidth,
    ]{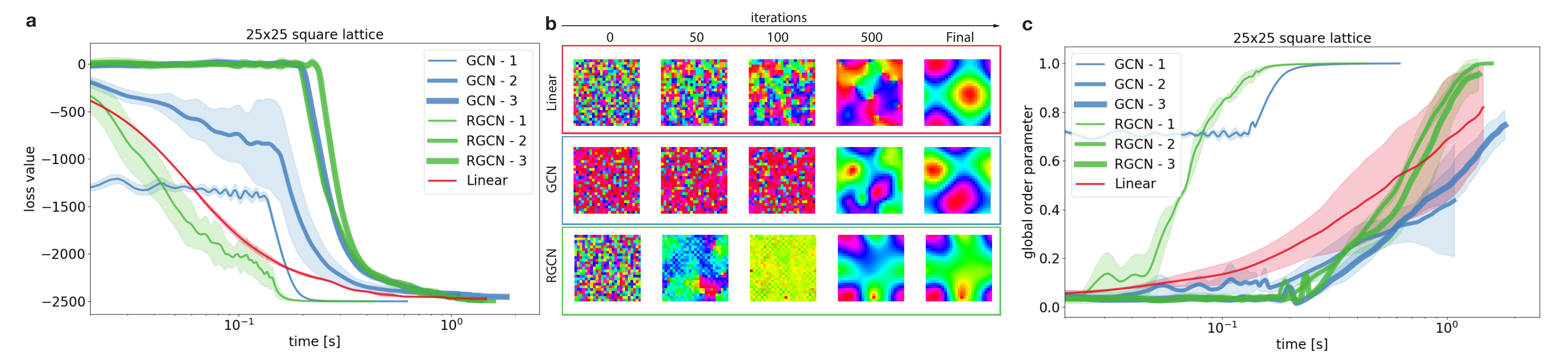}
    }
    \caption{Kuramoto model on a $25\times 25$ lattice (a) Loss over run time for different methods.
    (b) Evolution of $\vw$ over iterations.
    (c) Level of synchronization, measured by global order parameter $\rho$ over time.
    Neural reparametrization achieves the highest speedup.}
    \label{fig:Kuramoto_lattice}
\end{center}
\vskip -0.2in
\end{figure}

\begin{figure}
\centering
    \includegraphics[width=.65\linewidth,
    trim=0 .2in 0 .4in, 
    ]{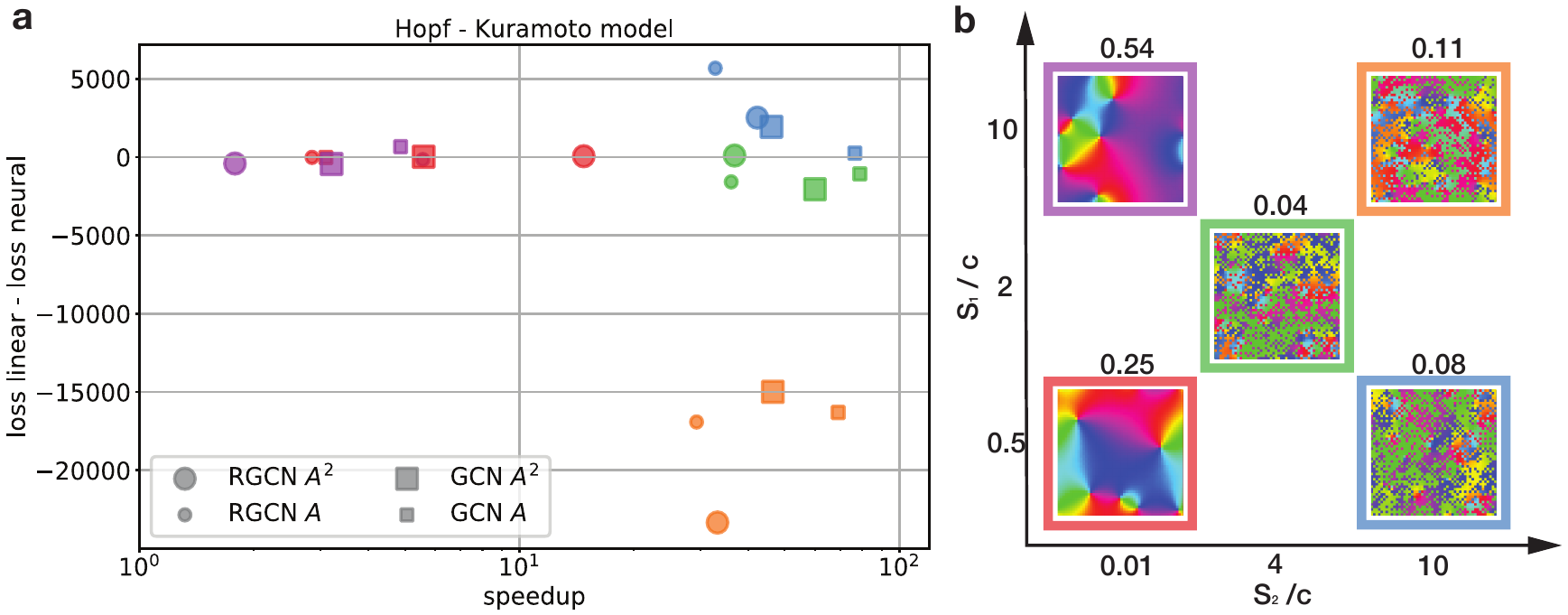}
    \caption{Hopf-Kuramoto model on a square lattice ($50\times 50$). a) Speedup in the final loss value difference function.
    Points color correspond to the regions of the phase diagram (b), also, the number above each phase pattern are the global order parameters.
    c) Coupled oscillator system.}
    \label{fig:hopf_kuramoto}
    \vskip -0.2in
\end{figure}


\textbf{Implementation.}
For early stages, we use a GCN with the aggregation function derived from the Hessian.
For the Kuramoto model, it is $\mathcal{H}_{ij}(0) = \ro^2\L/\ro\vw_i\ro \vw_j|_{\vw\to 0} = A_{ij}- \sum_k A_{ik}\delta_{ij} = -L_{ij} $ ($L= D-A$ being the Laplacian).
%
We use neural reparametrization in the first $100$ iterations and then switch to the original optimization (referred to as Linear) afterwards.
We experimented with three different graph structures: square lattice, circle graph, and tree graph.
The phases $\vw$ are randomly initialized between $0$ and $2\pi$ from a uniform distribution.
We let the models run until the loss  converges ($10$ patience steps for early stopping, $10^{-15}$ loss fluctuation limit).

\textbf{Results for the Kuramoto Model.}
Figure \ref{fig:Kuramoto_lattice} shows the results of Kuramoto model on a square lattice.
Additional results on  circle graph, and tree graph can be found in Appendix \ref{app:exp}.
Figure \ref{fig:Kuramoto_lattice} (a) shows that our method with one-layer GCN (GCN-1) and GCN with residual connection (RGCN-1) achieves significant speedup.
In particular, we found $3.6 \pm .5$ speed improvement for the lattice, $6.1 \pm .1$ for the circle graph and $2.7 \pm .3$ for tree graphs.
We also experimented with two layer (GCN/RGCN-2) and three layer (GCN/RGCN-3) GCNs.
As expected, the overhead of deeper GCN models slows down optimization and offsets the speedup gains.
Figure \ref{fig:Kuramoto_lattice} (b) visualizes the evolution of 
$\vw_i$ on a square lattice over iterations.
Although different GNNs reach the same loss value, the final solutions are quite different.
The linear model (without GNN) arrives at the final solution smoothly, while GNN models form dense clusters at the initial steps and reach an organized state before $100$ steps.
To quantify the level of synchronization, we measure a quantity $\rho$ known as the ``global order parameter'' (\cite{sarkar2021phase}):
$
    \rho = \frac{1}{N} \left|\sum_{j} e^{i\vw_j} \right|
$.
Figure \ref{fig:Kuramoto_lattice} (c) shows the convergence of the global order parameter over time. We can see that one-layer GCN and RGCN gives the highest amount of acceleration, driving the system to synchronization.

\textbf{Results for the Hopf-Kuramoto Model.} We report the comparison  on synchronizing more complex Hopf-Kuramoto dynamics.
According to the \citet{lauter2015pattern} paper, we identify two different main patterns on the phase diagram Fig. \ref{fig:hopf_kuramoto}
(b): ordered (small $s_2/c$, smooth patterns) and disordered (large $s_2/c$, noisy) phases ($c=1$).
In all experiments, we use the same lattice size $50 \times 50$, with the same stopping criteria ($10$ patience steps and $10^{-10}$ loss error limit) and switch between the Linear and GNN reparametrization after $100$ iteration steps.
Fig. \ref{fig:hopf_kuramoto} (a) shows the loss at convergence versus the speedup.
We compare different GCN models and observe that GCN with $A^2$ as the propagation rule achieves the highest speedup.
This is not surprising, as from \eqref{eq:Loss-Hopf-Kuramoto} the Hessian for HK contain $O(A^2)$ terms.
Also, we can see that we have different speedups in each region, especially in the disordered phases.
Furthermore, we observed that the Linear and GCN models converge into a different minima in a few cases.
However, the patterns of $\vw$ remain the same. \out{
Figure \ref{fig:hopf_kuramoto} (b) shows the level of ordering changes region by region.
If the global order parameter is closer to $0$, we have more of a disordered phase while the parameter is closer to $1$, meaning it is a more organized pattern.
}
Interestingly, in the disordered phase we observe the highest speedup
(Fig. \ref{fig:speedup_summary})

\subsection{Persistent homology}
Persistent homology \citet{edelsbrunner2008persistent} is an algebraic tool for measuring topological features of shapes and functions. Recently, it has found many applications in machine learning \cite{hofer2019connectivity, gabrielsson2020topology, birdal2021intrinsic}. Persistent homology is computable with linear algebra and robust to perturbation of input data \citep{otter2017roadmap}, see more details in Appendix \ref{app:homology}.
An example application of persistent homology is point cloud optimization \citep{gabrielsson2020topology, carriere2021optimizing}. As shown in Fig. \ref{fig:pers-hom-loss_vs_speedup} left, given a random point cloud $\vw$ that lacks any observable characteristics, we aim to produce persistent homological features by optimizing the position of data.
\begin{eqnarray}
\L(\vw) = -\sum_{p\in D}\| p -\pi_{\Delta} (p)\|^2_\infty + \sum_{i=1}^n \|\vw_i-r\|^2
\label{eqn:opt_persisten}
\end{eqnarray}
where $p\in D = \{(b_i,d_i)\}_{i\in I_k}$ denotes the homological features in the the persistence diagram $D$, consisting of all the pairs of birth $b_i$ and death $d_i$ filtration values of the set of k-dimensional homological features $I_k$. $\pi_{\Delta}$ is the projection onto the diagonal $\Delta$ and $\sum_i d(\vw_i,S)$ constrains the points within the a square centered at the origin with length $2r$ (denote $r$ as range of the point cloud).
\cite{carriere2021optimizing} optimizes the point cloud positions directly with gradient-based optimization (we refer to it as ``linear'').

\begin{figure}
\vskip -0.2in
\centering
    a \hfill b \hfill c \hfill ~\\
    \includegraphics[width=\linewidth,
    trim=0 0 0 .28in, clip
    ]{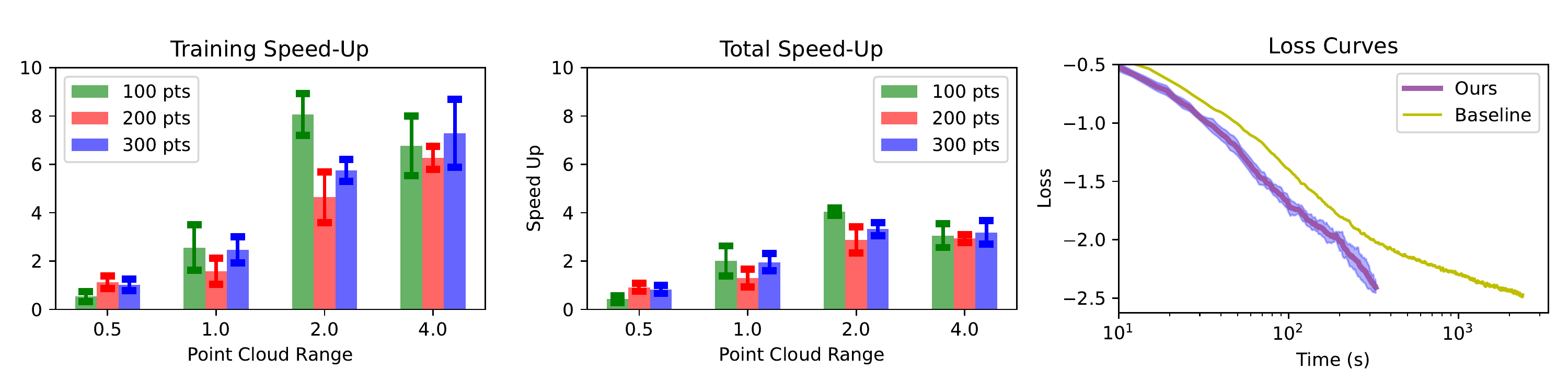}\\
    \includegraphics[width=.49\linewidth ]{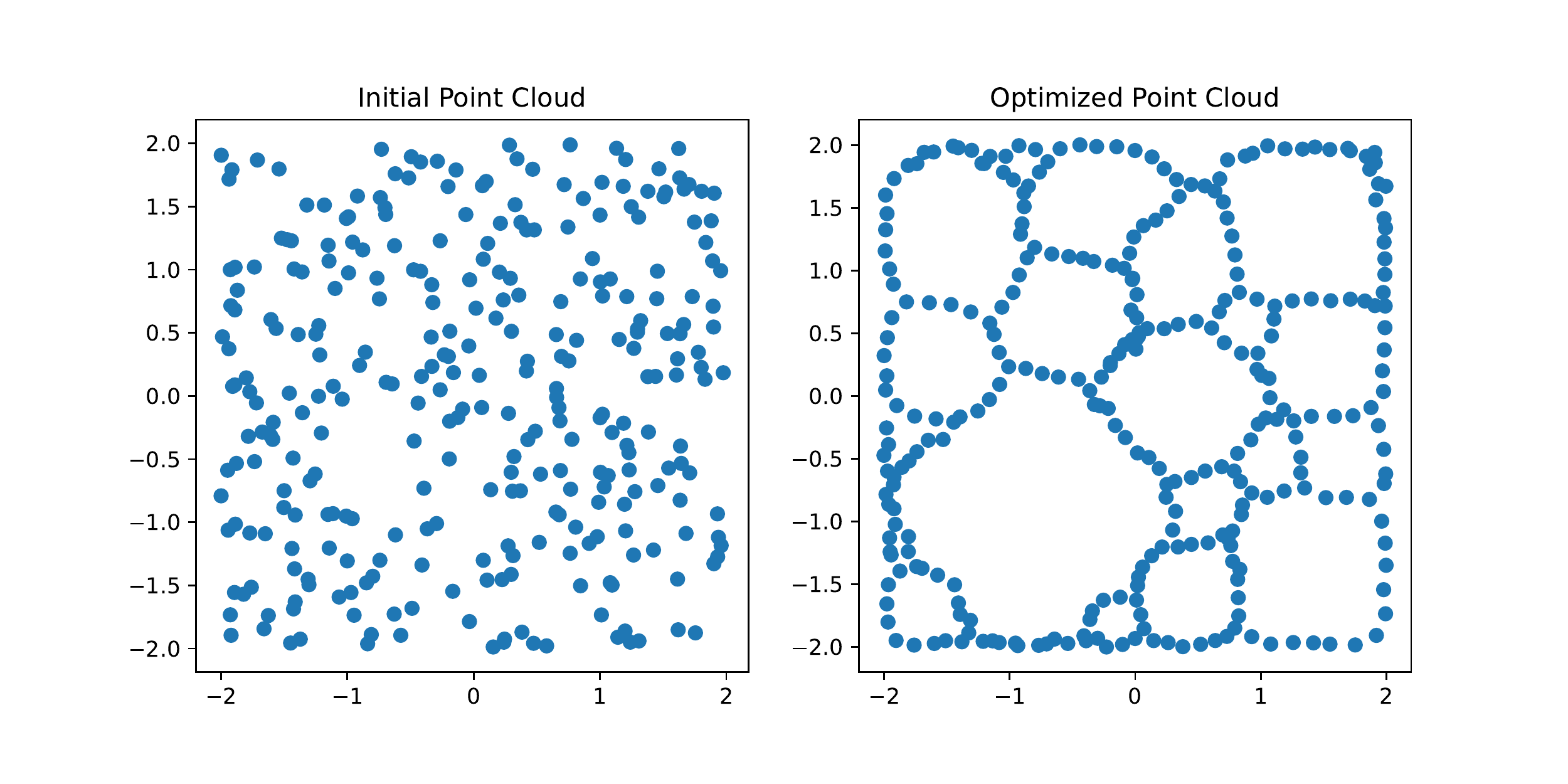}\includegraphics[width=.5\linewidth,trim=0 .35in 0 0, 
    ]{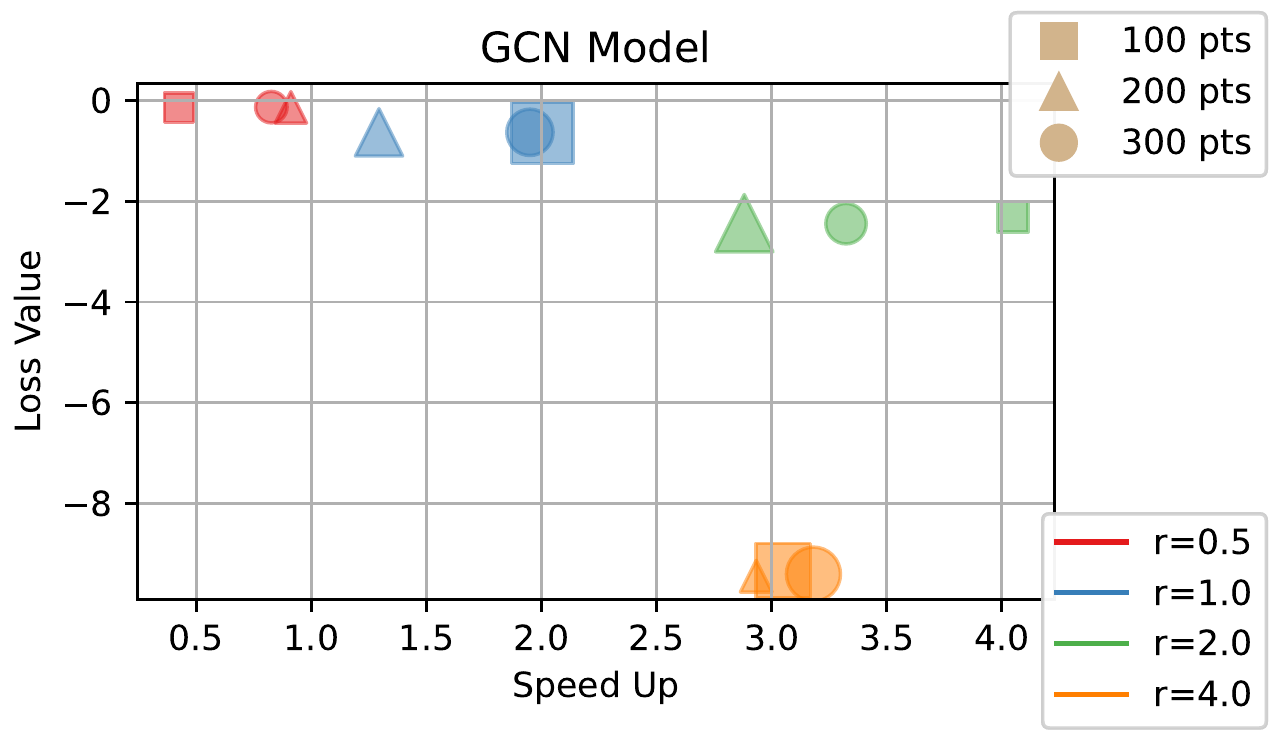}
    \vskip -1.5in
    d \hfill e \hfill f \hfill ~\hfill ~ \\
    \vskip 1.5in
    \caption{{Speedup:} a,b) Training and total time speedup; c) GCN speeds up convergence.
    d) Initial random point cloud and e) optimized point cloud.
    f) Loss vs speed-up of GCN model. The final loss depends only on the point cloud range, while the speedup is affected by both range and size.
    }
    \label{fig:pers-hom-GCN_1L_speedup}
    \label{fig:pers-hom-loss_vs_speedup}
\vskip -0.2in
\end{figure}

\out{
\begin{figure}[b]
    \centering
    \includegraphics[width=.49\linewidth ]{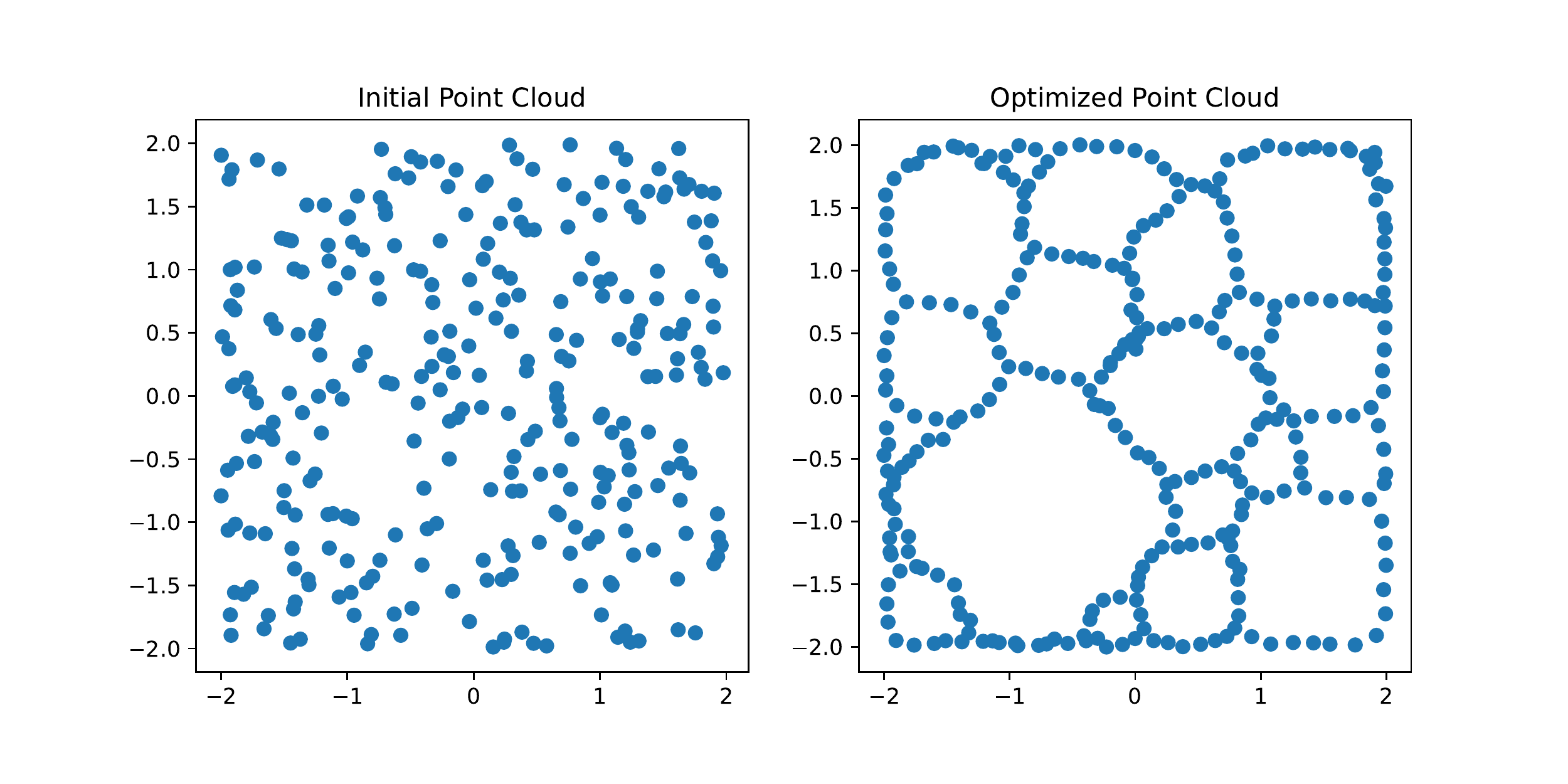}\includegraphics[width=.5\linewidth]{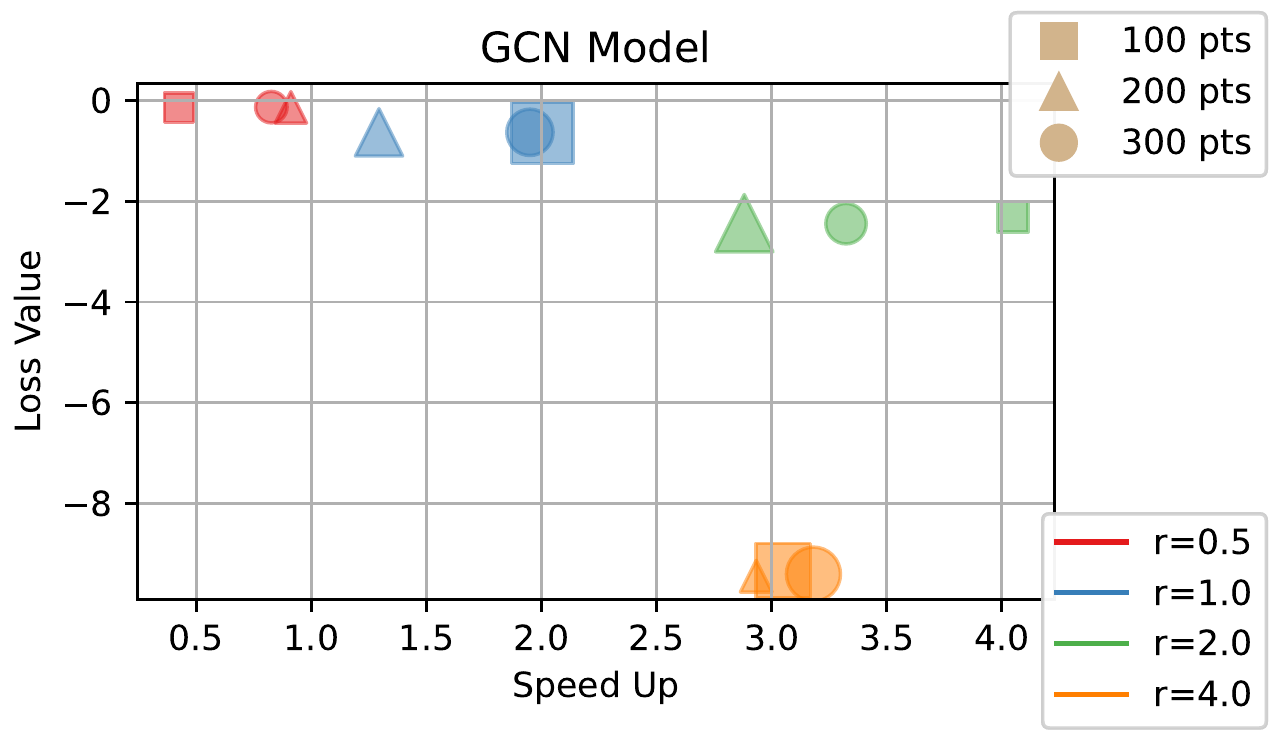}
 \caption{Left: Initial random point cloud and optimized point cloud. Right: Loss vs speed-up of GCN model. The final loss depends only on the  point cloud range, while the speedup is affected by both range and size.}
  \label{fig:pers-hom-loss_vs_speedup}
\vskip -0.2in
\end{figure}
}

\textbf{Implementation.}
We used the same Gudhi library for computing persistence diagram as \cite{gabrielsson2020topology, carriere2021optimizing}.
The run time of learning persistent homology is dominated by computing persistence diagram in every iteration, which has the time complexity of $O(n^3)$. 
Thus, the run time per iteration for GCN model and linear model are very similar, and we find that the GCN model can reduce convergence time by a factor of $\sim4$ (Fig. \ref{fig:pers-hom-GCN_1L_speedup},\ref{fig:pers-hom-loss_vs_speedup}).
We ran the experiments for point cloud of $100$,$200$,$300$ points, with ranges of $0.5$,$1.0$,$2.0$,$4.0$.
The hyperparameters of the GCN model are kept constant, including network dimensions.
The result for each setting is averaged from 5 consecutive runs.

\textbf{Results for Persistent Homology.}
Figure \ref{fig:pers-hom-loss_vs_speedup} right shows that the speedup of the GCN model is related to point cloud density.
In this problem, the initial position of the point cloud determines the topology features.
Therefore, we need to make sure the GCN models also yield the same positions as used in the linear model. Therefore, we first run a ``training" step, where we use MSE to match the $\vw(\theta)$ or GCN to the initial $\vw$ used in the linear model.
Training converges faster as the point cloud becomes more sparse, but the speedup gain saturates as point cloud density decreases.
On the other hand, time required for initial point cloud fitting increases significantly with the range of point cloud.
Consequently, the overall speedup peaks when the range of point cloud is around 4 times larger than what is used be in \cite{gabrielsson2020topology, carriere2021optimizing}, which spans over an area 16 times larger.
Further increase in point cloud range causes the speedup to drop as the extra time of initial point cloud fitting outweighs the reduced training time.
The loss curve plot in Fig. \ref{fig:pers-hom-GCN_1L_speedup} shows the convergence of training loss of the GCN model and the baseline model in one of the settings when GCN is performing well.
Fig. \ref{fig:pers-hom-loss_vs_speedup} shows the initial random point cloud and the output from the GCN model.
In the Appendix  \ref{app:homology}, we included the results of GCN model hyperparameter search and a runtime comparison of the GCN model under all experiment settings.


%% file: secs/conclusion.tex
We propose a novel neural reparametrization scheme to    accelerate a large class of graph optimization problems. By reparametrizing the optimization problem with a graph convolutional network, we can modify the geometry of the loss landscape and obtain the maximum speed up. The effect of neural reparametrization mimics the behavior of adaptive gradient methods. A linear reparametrization of GCN recovers the optimal learning rate from AdaGrad.
The aggregation function of the GCN is constructed from the gradients of the loss function and reduces to the Hessian in early stages of the optimization.
We demonstrate our method on optimizing heat diffusion, network synchronization problems and persistent homology of point clouds. Depending on the experiment, we obtain a best case speedup that ranges from 5 - 100x.

One limitation of the work is that the switching from neural reparameterization to the original optimization stage is still ad-hoc. Further insights into the learning dynamics of the optimization are needed. Another interesting direction is to extend our method to stochastic dynamics, which has close connections with energy-based generative models.

%% file: secs/appendix.tex
\input{secs/appendices/method-app}
\input{secs/appendices/theory-app}


\input{secs/appendices/theory-ext-app}

\input{secs/appendices/experiment-app}

%% file: secs/appendices/theory-app.tex

%% file: secs/appendices/theory-ext-app.tex
\section{Extended Derivations}

\subsection{Estimating $G_t$ at $t\to0$ \label{ap:G-estimate}}
$G = \E[\del\L\del \L^T]$ can be written in terms of the moments of the random variable $\mW=\{\vw\} $ using the Taylor expansion of $\L$ around $\vw_i=0$ plugged into
$G = \E[\del\L\del \L^T]$
\begin{align}
    {\ro\L(\vw)\over \ro \vw_i}& \approx \left.\sum_{k=0}^\infty{1\over k!}  \br{\vw^T{\ro \over \ro \vv}}^k {\ro \L(\vv) \over \ro \vv_i}\right|_{\vv \to0} 
    \label{eq:L-Taylor}
\end{align}
\begin{align}
    G_{ij}(t) 
    &= \left.\sum_{p,q=1}^\infty {1\over p!q!} 
    \sum_{\{i_a\}} 
    \E_P\br{\vw_{i_1} \dots \vw_{i_{p+q}}}
    {\ro^{p+1} \L \over \ro \vv_{i_1}\dots \ro \vv_{i_p}\ro \vv_i }
    {\ro^{q+1} \L \over \ro \vv_{i_{p+1}}\dots \ro \vv_{i_{p+q}}\ro \vv_j }\right|_{\vv \to0}
    \label{eq:M-moments}
\end{align}
where the sum over $\{i_a\}$ means all choices for the set of indices $i_1\dots i_{p+q}$. 
\Eqref{eq:M-moments} states that we need to know the order $p+q$ moments $\E_P[\vw_{i_1}\dots \vw_{i_{p+q}}]$ of $\mW $ to calculate $G$.
This is doable in some cases. 
For example, we can use a normal distribution is used to initialize $\vw(t=0)\in \R^n $. 
The local minima of $\L$ reachable by GD need to be in a finite domain for $\vw$. Therefore, we can always rescale and redefine $\vw$ such that its domain satisfies $\|\vw\|_2 \leq 1 $.  
This way we can initialize with normalized vectors $\vw(0)^T\vw(0) = 1$ using   $\sigma=1/\sqrt{n}$ and get 
\begin{align}
    \vw_{(i)j}(0) &= 
    \mathcal{N}\pa{0,n^{-1/2}}, 
    & \\
    \E_P\br{ \prod_{a=1}^p \vw_{i_a}} & = \delta_{i_1\dots i_p} {(p-1)!!\over  n^{p/2}}\mathrm{even}(p) 
    \label{eq:init-w-normal}
\end{align}
where $\delta_{i\dots j}$ is the Kronecker delta, $(p-1)!! = \prod_{k=0}^{[p/2]} (p-1-2k)$ and $\mathrm{even}(p)=1$ if $p$ is even and $0$ if it's odd. 
When the number of variables $n\gg 1$ is large, order $p>2$ moments are suppressed by factors of $n^{-p/2}$. 
Plugging \eqref{eq:init-w-normal} into \eqref{eq:M-moments} and 
defining the Hessian $\mathcal{H}_{ij}(\vw)\equiv \ro^2 \L(\vw)/\ro \vw_i \ro \vw_j$ 
we have 
\begin{align}
    G_{ij}(t=0)& 
    = {1\over n}\left.\br{\mathcal{H}^2}_{ij}\right|_{\vw \to 0} + O(n^{-2}). 
    \label{eq:M-Hessian2}
\end{align}
\out{
Thus, with normal initialization, the expected convergence rate is the mean of the squares of eigenvalues of the Hessian 
\begin{align}
    \E_P\br{{d\L\over dt}}& \approx \left. {1\over n} \Tr\br{\mathcal{H}^2}\right|_{\vw \to 0} 
\end{align}
}
Here the assumption is that the derivatives $\ro_\vw^p \L $ are not $n$ dependent after rescaling the domain such that $\|\vw\|\leq 1$. 
In the experiments in this paper this condition is satisfied.

\subsection{Computationally efficient implementation \label{ap:implement} } 
We focus on speeding up the early stages, in which we use the Hessian $\mathcal{H}$ is $ G(0)\approx \mathcal{H}^2/n$ as in Newton's method. 
To control the learning rate $\eta $ in Newton's method $\ro_t \vw = -\eta \mathcal{H}^{-1} \del \L $, we need to work with the normalized matrix $\mathcal{H}/h_{max}$. 
Next, we want an approximate Jacobian $J\sim  (\mathcal{H}/h_{max})^{-1/2}$ written as a expansion. 
To ensure we have a matrix whose eigenvalues are all less than 1, we work with $\mH\equiv (1-\xi)\mathcal{H}/h_{max}$, where $h_{max}=\sqrt{\lambda_{max}}$ and $\xi\ll 1$. 
To get an $O(qn^2)$ approximation for 
$J$
we can take the first $q$ terms in the binomial expansion as  
\begin{align}
    \mH^{-1/2} &\approx I- {1\over 2} (I-\mH) -{3\over 4} (I-\mH)^2 + \dots.
    \label{eq:M14-rough}
\end{align}
Since $\mH$ is positive semi-definite and its largest eigenvalue is $1-\xi <1 $, the sum in \eqref{eq:M14-rough} can be truncated after $q\sim O(1)$ terms. 
The first $q$ terms of 
\eqref{eq:M14-rough} can be implemented as a $q$ layer GCN with aggregation function $f(\mathcal{H}) = I-\mathbf{H}$ and residual connections.
We choose $q$ to be as small as $1$ or $2$, as larger $q$ may actually slow down the optimization. 
Note that computing $f(\mathcal{H})^q\theta $ is $O(qn^2)$ because we do not need to first compute $f(\mathcal{H})^q$ (which is $O(qn^3)$). 
Instead, we use the forward pass through the GCN layers, which with linear activation is case is $\vv_{i+1} = f(\mathcal{H})\vv_i$ ($O(n^2)$). 
This way, $q$ layers with $\vv_1 = \theta$ implements $f(\mathcal{H})^q\theta $ with $O(qn^2h)$ for $\theta \in \R^{n\times h}$ ($h \ll n$). 

\paragraph{GCN Aggregation Rule}
To evaluate $\mathbf{H}$ we need to estimate the leading eigenvalue $h_{max}$. 
In the graph problems we consider we have $\mathcal{H}\approx L = D-A$, where $L$ is the graph Laplacian and $D_{ij} = \delta_{ij} \sum_k A_{ik}$ is the degree matrix. 
In this case, instead of dividing by $h_{max}$ an easy alternative is
\begin{align}
    \mathbf{H} = D^{-1/2} L D^{-1/2} = I - A_s 
\end{align}
where $A_s = D^{-1/2} A D^{-1/2}$, and we chose this symmetrized form instead of $D^{-1}L$ because the Hessian is symmetric. 
When the edge weights are positive and degrees of nodes are similar $D_{ii} \approx k$ (e.g. mesh, lattice, RGG, SBM), we expect $h_{max} \sim O(k) $.  
This is because when degrees are similar $L\approx kI-A $.
$L$ is PSD as $v^T Lv = \sum_{ij} A_{ij}(v_i-v_j)^2 \geq 0$. 
Therefore, when $L= kI-A$ the largest eigenvalue $\alpha_1$ of $A$ is bounded by $\alpha_1 \leq k$. 
When the graph is grid-like, its eigenvectors are waves on the grid and the eigenvalues are the Fourier frequencies, $\pm k/ m$ for $m\in [1\to n/2]$. 
This bounds $h_{max}\leq k-(-k) = 2k$. 
and when the graph is mostly random, its spectrum follows the Wigner semi-circle law, which states most eigenvalues are concentrated near zero, hence $h_{max} \approx k $.
This goes to say that choosing $\mathbf{H} = I- A_s $ should be suitable for numerical stability, as the normalization in $A_s$ is comparable to $\mathcal{H}/h_{max}$. 
\out{
This is a crude approximation, but it serves our goal by being computationally efficient. 
Note that, while computing $M^q$ is $O(qn^3)$, because $M\mathbf{1}$ is $O(n^2)$, computing $M^q\mathbf{1}$ is $O(qn^2)$.  
\nd{Todo: Check if $[M,K]=0$ yields maximum speedup.}
}

\subsection{Network Synchronization} 

The HK model's dynamics are follows:
\begin{align}
    \frac{d\vw_i}{dt} &= c \sum_jA_{ji}\br{\cos\Delta_{ij} - s_1\sin\Delta_{ij}} \cr
    &+ s_2\sum_{k,j}  A_{ij} A_{jk} \br{\sin\pa{\Delta_{ji}+\Delta_{jk}} - \sin\pa{\Delta_{ji}-\Delta_{jk}}}  \cr
    &+A_{ij} A_{ik} \sin\pa{\Delta_{ji}+\Delta_{ki}}. 
    \label{eq:Hopf-Kuramoto-GF}
\end{align}
where $c$, $s_1,s_2$ are the governing parameters. 
\Eqref{eq:Hopf-Kuramoto-GF} is the GF equation for the 
following loss function (found by integrating \Eqref{eq:Hopf-Kuramoto-GF}):
\begin{align}
    \L(\vw) &= {c\over \eps} \sum_{i,j} A_{ji}\br{ \sin\Delta_{ij} + s_1\cos\Delta_{ij}} 
    \cr 
    &+ {s_2\over 2\eps}\sum_{i,k,j} A_{ij} A_{jk}  \big[\cos\pa{\Delta_{ji}+\Delta_{jk}}\cr 
    &+\cos\pa{\Delta_{ji}-\Delta_{jk}}\big] \label{eq:Loss-Hopf-Kuramoto-app}
\end{align}


%
 \begin{figure}[t!]
  \begin{center}
    \includegraphics[width=1\linewidth]{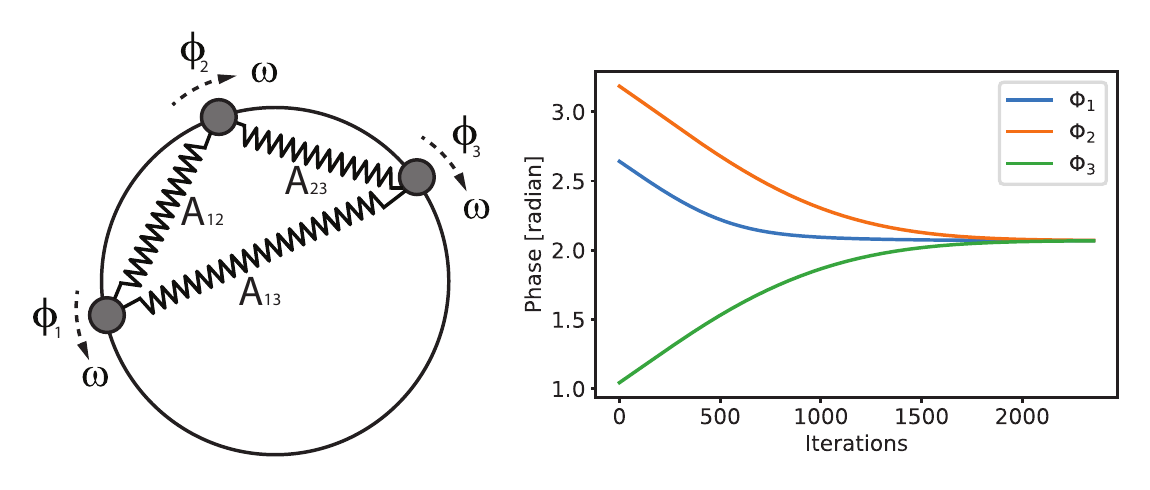}
  \end{center}
        \vskip -0.2in
    \caption{Illustration of Kuramoto oscillators. Phase changes during the synchronization.}
      \label{fig:kuramoto_example}
\end{figure}

%% file: secs/appendices/experiment-app.tex
\section{Experiment details and additional results}
\label{app:exp}

\subsection{Network Synchronization}

Network synchronization \citep{pikovsky2003synchronization} optimizes a network of coupled oscillators until they reach  at the same frequency, known as synchronization. Kuramoto model \citet{kuramoto1975self,kuramoto1984chemical} are widely used for synchronization problems, which have  profound impact on engineering, physics and machine learning \cite{schnell2021half}.

As shown in Fig. \ref{fig:kuramoto_example},  Kuramoto model describes the behavior of a large set of coupled oscillators. Each oscillator is defined by an angle $\theta_i = \omega_i t+ \vw_i$, where $\omega_i$ is the frequency and $\vw_i$ is the phase. 
We consider the case where $\omega_i=0$.   The coupling strength is represented by a graph  $A_{ij}\in \R$.
Defining $\Delta_{ij}\equiv \vw_i-\vw_j$,
the dynamics of the phases $\vw_i(t)$ in the Kuramoto model follows the following equations:
\begin{align}
    \frac{d\vw_i}{dt} = -\eps \sum_{j=1}^{n} A_{ji} \sin{\Delta_{ij}},   \
    \L(\vw) = \sum_{i,j=1}^{n} A_{ji} \cos\Delta_{ij}.
    \label{eq:Kuramoto-app}
\end{align}
Our goal  is to minimize the phase drift $d\vw_i/dt$ such that   the oscillators are synchronized. 
We further consider a more general version of the Kuramoto model: Hopf-Kuramoto (HK)  model  \citet{lauter2015pattern}, which  includes second-order interactions.

We experiment with both the Kuramoto model and Hopf-Kuramoto model. 
Existing numerical methods directly optimize the loss $\L(\vw)$ with gradient-based algorithms, which we refer to as \textit{linear}. 
We apply our method to reparametrize the phase variables $\vw$  and  speed up  convergence towards synchronization.

\textbf{Implementation.}
For early stages, we use a GCN with the aggregation function derived from the Hessian 
which for the Kuramoto model simply becomes $\mathcal{H}_{ij}(0) = \ro^2\L/\ro\vw_i\ro \vw_j|_{\vw\to 0} = A_{ij}- \sum_k A_{ik}\delta_{ij} = -L_{ij} $, where $L= D-A$ is the graph Laplacian of $A$. 
We found that NR in the early stages of the optimization gives more speed up. 
We implemented the hybrid optimization described earlier, where we reparemtrize the problem in the first $100$ iterations and then switch to the original \textit{linear} optimization for the rest of the optimization. 

We experimented with three Kuramoto oscillator systems with different coupling structures: square lattice, circle graph, and tree graph. 
For each system, the phases are randomly initialized between $0$ and $2\pi$ from uniform distribution. 
We let the different models run until the loss  converges ($10$ patience steps for early stopping, $10^{-15}$ loss fluctuation limit).

\subsection{Kuramoto Oscillator}
\label{app:kuramoto}


\begin{figure}
 \centering
 \includegraphics[width=\linewidth,
 ]{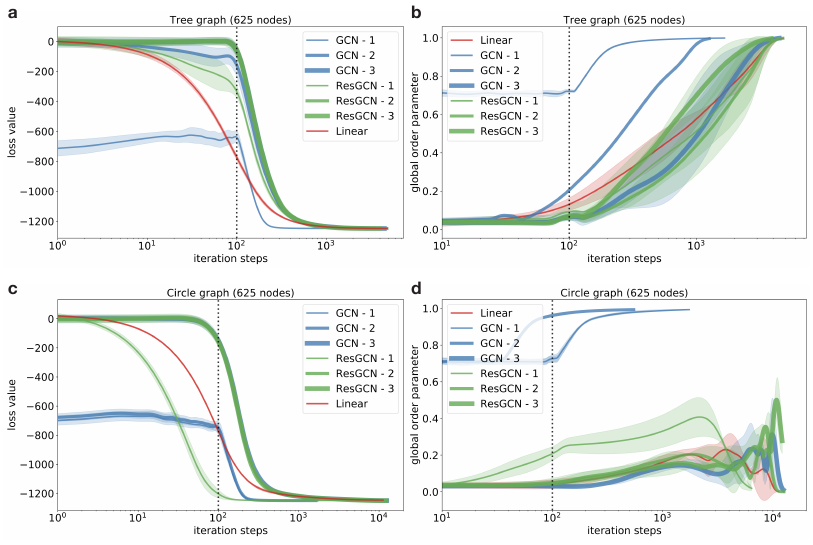}
 \caption{Kuramoto model on (a,b) tree and (c,d) circle graph. a,c) loss curves evolution, while b,d) global order parameter in each iteration steps}
 \label{fig:Kuramoto_tree_circle}
\end{figure}

\begin{figure}
 \centering
 \includegraphics[width=\linewidth,
 ]{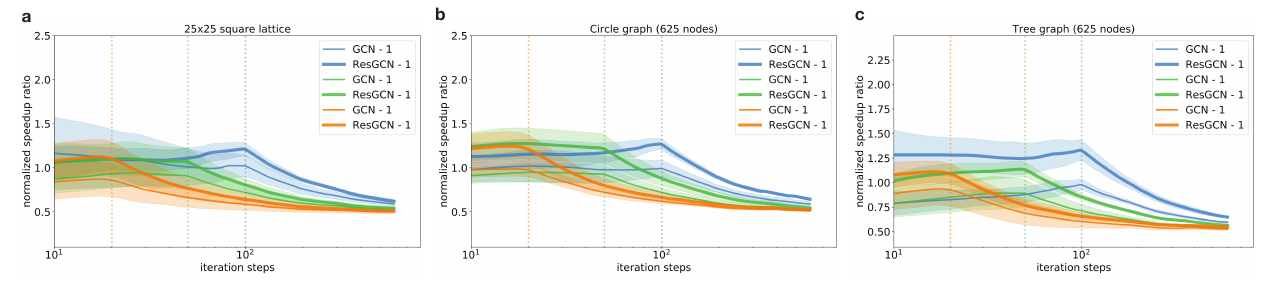}
 \caption{Kuramoto-model on different graph structures. We use the GCN model and switch to the Linear model after different iteration steps: orange - 20 steps, green - 50 steps, and blue - 100 steps. The plot shows that each GCN iteration step takes longer.  }
 \label{fig:Kuramoto_gcn_switch}
\end{figure}

\begin{figure}
 \centering
 \includegraphics[width=\linewidth,
 ]{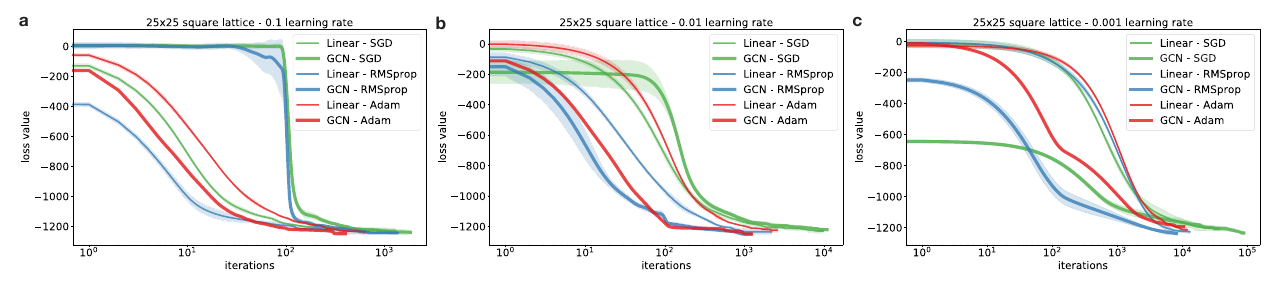}
 \caption{Kuramoto-model on $25 \times 25$ square lattice. We use the GCN and Linear model with different optimizers and learning rates. The Adam optimizer in all cases over-performs the other optimizers.}
 \label{fig:Kuramoto_lr_opt}
\end{figure}

\begin{figure}
    \centering
    \includegraphics[width=\linewidth,
    ]{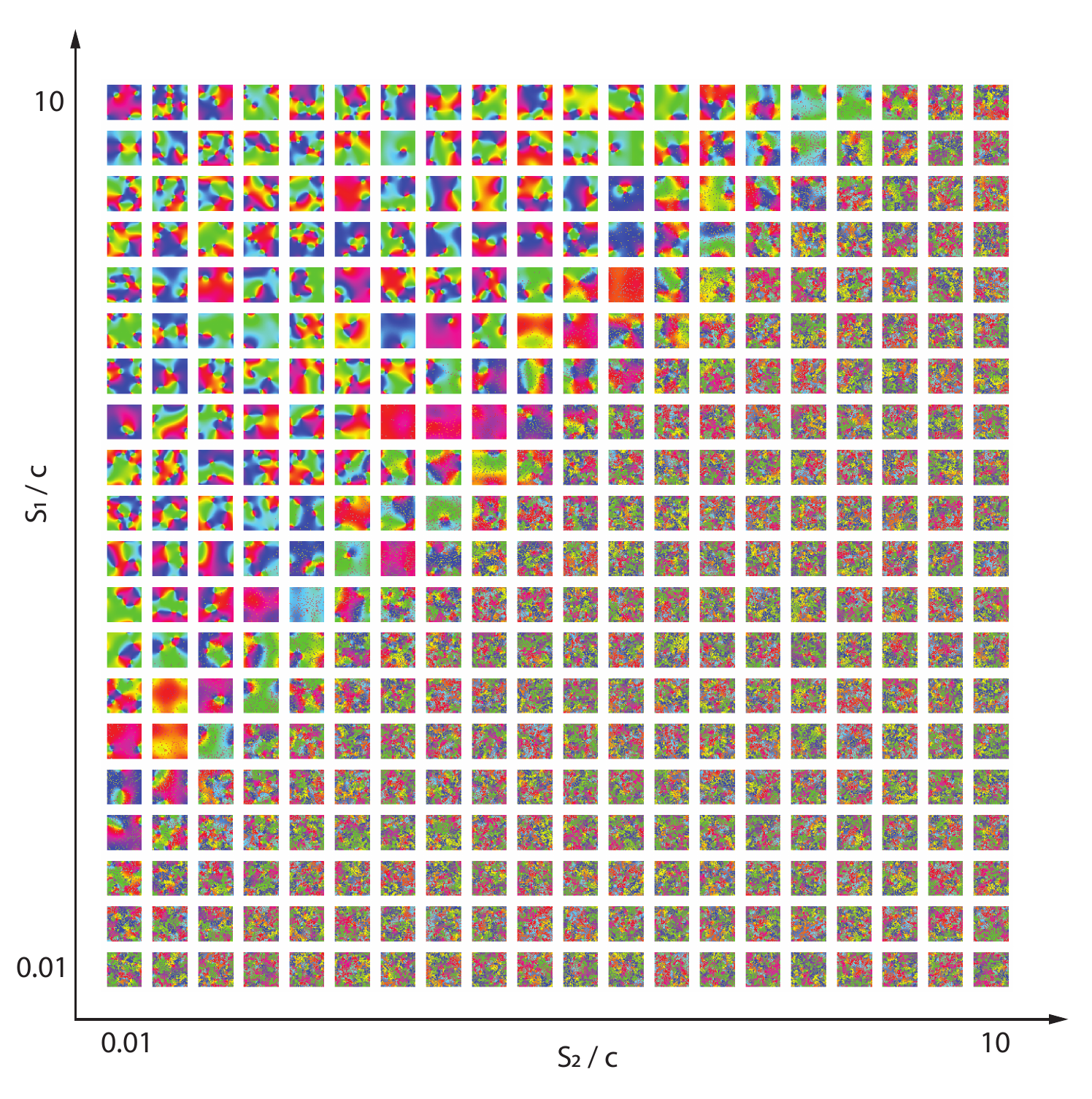}
    \caption{Hopf-Kuramoto model on a square lattice ($50\times 50$) - phase pattern. We can distinguish two main patterns - organized states are on the left part while disorganized states are on the right part of the figure. In the experiments $c = 1$ for the simplicity.}
    \label{fig:HK-pattern}
\end{figure}

\paragraph{Relation to the XY model}
The loss \eqref{eq:Kuramoto} is also identical to the Hamiltonian (energy function) of a the classical XY model (\cite{kosterlitz1973ordering}), a set of 2D spins $s_i$ with interaction energy given by $\L= \sum_{i,j} A_{ij} s_i\cdot s_j = \sum_{i,j} A_{ji} \cos(\vw_j - \vw_i) $. 
In the XY model, we are also interested in the minima of the energy. 

\paragraph{Method}
In our model, we first initialize random phases between $0$ and $2\pi$ from a uniform distribution for each oscillator in a $h$ dimensional space that results in $N \times h$ dimensional vector $N$ is the number of oscillators. 
Then we use this vector as input to the GCN model, which applies $ D^{-1/2}\hat{A}D^{-1/2}$ , $\hat{A} = A + I$ propagation rule, with LeakyRelu activation. 
The final output dimension is $N \times 1$, where the elements are the phases of oscillators constrained between $0$ and $2\pi$. 
In all experiments for $h$ hyperparameter we chose $10$. Different $h$ values for different graph sizes may give different results. Choosing large $h$ reduces the speedup significantly.
We used Adam optimizer by $0.01$ learning rate. 

\subsection{MNIST image classification}
\label{app:mnist}
Here, we introduce our reparametrization model for image classification on the MNIST image dataset. First, we test a simple linear model as a baseline and compare the performance to the GCN model.  We use a cross-entropy loss function, Adam optimizer with a 0.001 learning rate in the experiments, and Softmax nonlinear activation function in the GCN model. We train our models on $100$ batch size and $20$ epochs. In the GCN model, we build the $\mH$ matrix  (introduced in the eq. \ref{eq:M14-rough}) from the covariance matrix of images and use it as a propagation rule. In the early stages of the optimization, we use the GCN model until a plateau appears on the loss curve then train the model further by a linear model. We found that the optimal GCN to linear model transition is around 50 iterations steps. Also, we discovered that wider GCN layers achieve better performance; thus, we chose $500$ for the initially hidden dimension. According to the previous experiments, the GCN model, persistent homology (\ref{app:homology}), and Kuramoto (\ref{app:kuramoto}) model speedups the convergence in the early stages (Fig. \ref{fig:mnist})

\begin{figure}
    \centering
    \includegraphics[width=\linewidth,
    ]{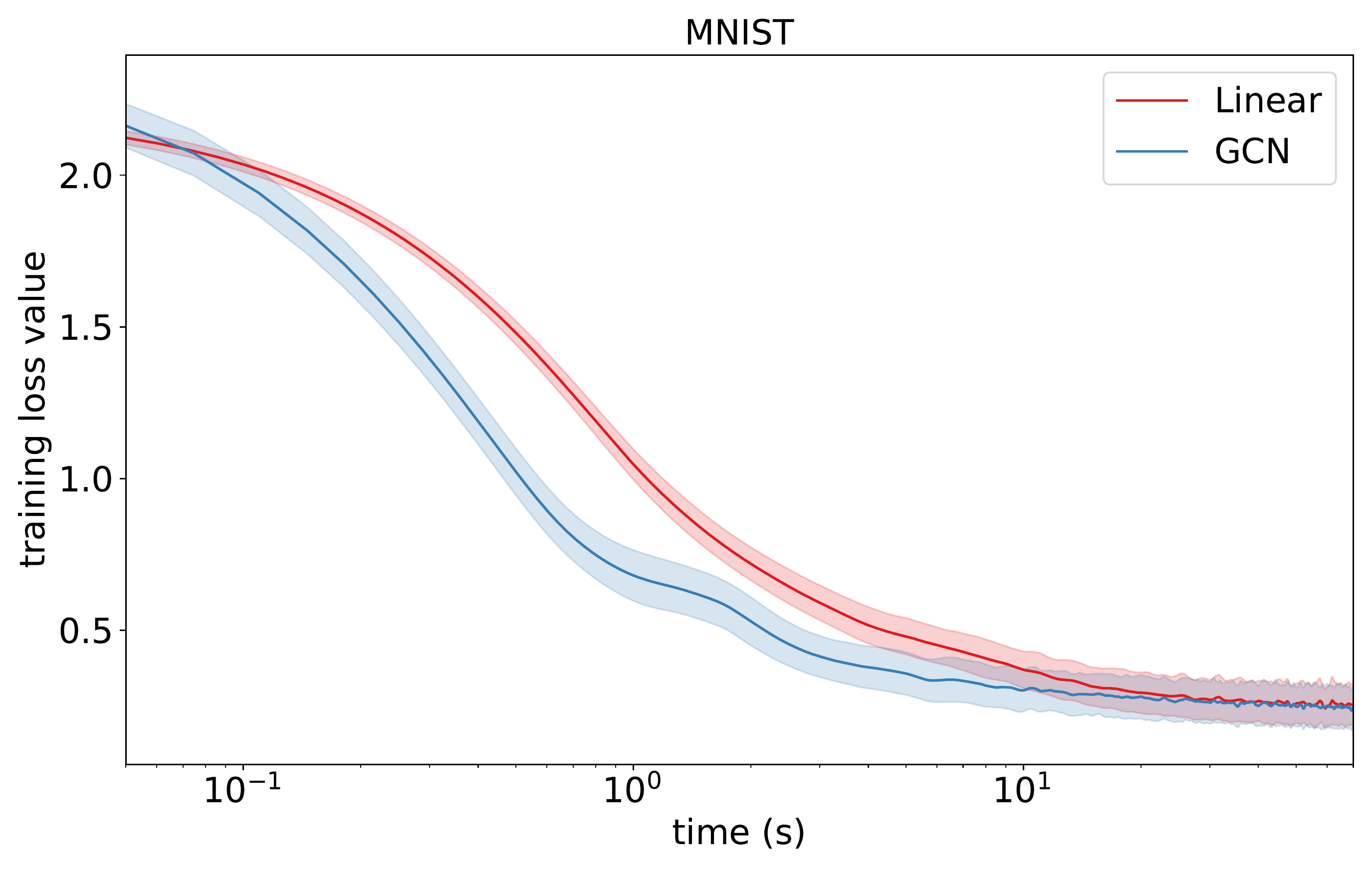}
    \caption{MNIST image classification. Red curve shows a linear model ($92.68 \%$ accuracy). Blue curve is the reparametrized model, where at step $50$ we switch from GCN to Linear model ($92.71 \%$ accuracy). }
    \label{fig:mnist}
\end{figure}

\subsection{Persistent Homology}
\label{app:homology}
\paragraph{Overview.} 
Homology describes the general characteristics of data in a metric space, and is categorized by the order of its features. Zero order features correspond to connected components, first order features have shapes like "holes" and higher order features are described as "voids".  

A practical way to compute homology of a topological space is through forming simplicial complexes from its points. This enables not only fast homology computation with linear algebra, but also approximating the topological space with its subsets. 

In order to construct a simplicial complex, a filtration parameter is needed to specify the scope of connectivity. Intuitively, this defines the resolution of the homological features obtained. A feature is considered "persistent" if it exists across a wide range of filtration values. In order words, persistent homology seeks features that are scale-invariant, which serve as the best descriptors of the topological space.

There are different ways to build simplicial complexes from given data points and filtration values. Czech complex, the most classic model, guarantees approximation of a topological space with a subset of points. However, it is computationally heavy and thus rarely used in practice. Instead, other models like the Vietoris-Rips complex, which approximates the Czech complex, are preferred for their efficiency \citep{otter2017roadmap}. Vietoris-Rips complex is also used in the point cloud optimization experiment of ours and \cite{gabrielsson2020topology, carriere2021optimizing}.

\paragraph{Algorithm Implementation.} 
Instead of optimizing the coordinates of the point cloud directly, we reparameterize the point cloud as the output of the GCN model. To optimize the network weights, we chose identity matrix with dimension of the point cloud size as the fixed input.

To apply GCN, we need the adjacency matrix of the point cloud. Even though the point cloud does not have any edges, we can manually generate edges by constructing a simplicial complex from it. The filtration value is chosen around the midpoint between the maximum and minimum of the feature birth filtration value of the initial random point cloud, which works well in practice.

Before the optimization process begins, we first fit the network to re-produce the initial random point cloud distribution. This is done by minimizing MSE loss on the network output and the regression target. 

Then, we begin to optimize the output with the same loss function in \cite{gabrielsson2020topology, carriere2021optimizing}, which consists of topological and distance penalties. The GCN model can significantly accelerates convergence at the start of training, but this effect diminishes quickly. Therefore, we switch the GCN to the linear model once the its acceleration slows down. We used this hybrid approach in all of our experiments.

\paragraph{Hyperparameter Tuning.} 


We conducted extensive hyperparameter search to fine tune the GCN model, in terms of varying hidden dimensions, learning rates and optimizers. We chose the setting of 200 point cloud with range 2.0 for all the tuning experiments. 

Fig. \ref{fig:GCN_vary_dims} shows the model convergence with different hidden dimensions. We see that loss converges faster with one layer of GCN instead of two. Also, convergence is delayed when the dimension of GCN becomes too large. Overall, one layer GCN model with $h1,h2=8,6$ generally excels in performance, and is used in all other experiments.

Fig. \ref{fig:Final_Loss_Heatmap},\ref{fig:Training_Speedup_Heatmap},\ref{fig:Total_Speedup_Heatmap} shows the performance of the GCN model with different prefit learning rates, train learning rates and optimizers. From the results, a lower prefit learning rate of 0.01 combined with a training learning rate below 0.01 generally converges to lower loss and yields better speedup. For all the optimizers, default parameters from the Tensorflow model are used alongside varying learning rates and the same optimizer is used in both training and prefitting. Adam optimizer is much more effective than RSMProp and SGD on accelerating convergence. For SGD, prefitting with learning rate 0.05 and 0.1 causes the loss to explode in a few iterations, thus the corresponding results are left as blank spaces.

\begin{figure}[htb]
    \centering
    \includegraphics[width=\linewidth
    ]{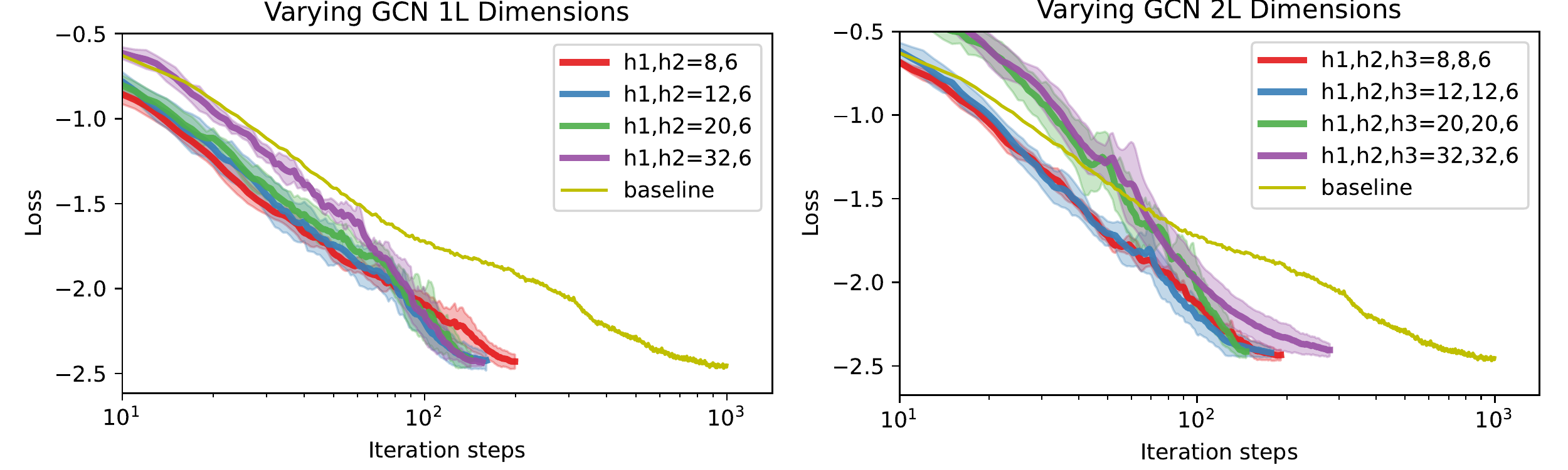}
    \caption{GCN Hyperparameter Comparison. We recommend using one layer GCN model with $h1,h2=8,6$.}
    \label{fig:GCN_vary_dims}
    
    \vspace{0.25in}
    \centering
    \includegraphics[width=\linewidth]{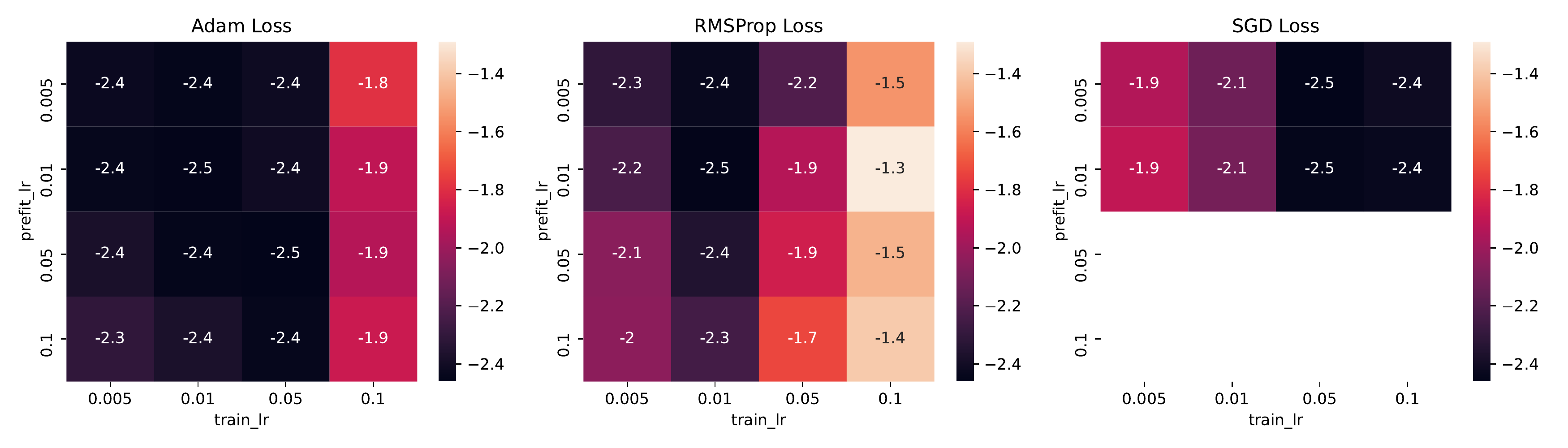}
    \caption{Converged loss of different learning rates and optimizers.}
    \label{fig:Final_Loss_Heatmap}
    
    \vspace{0.25in}
    \centering
    \includegraphics[width=\linewidth]{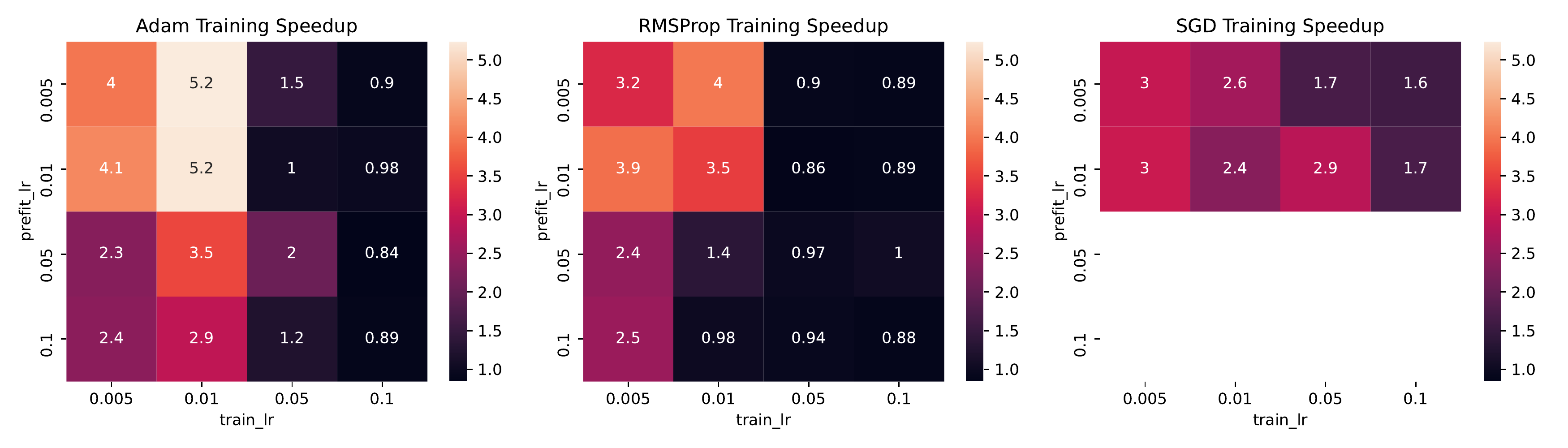}
    \caption{Training Speedup of different learning rates and optimizers.}
    \label{fig:Training_Speedup_Heatmap}
    
    \vspace{0.25in}
    \centering
    \includegraphics[width=\linewidth]{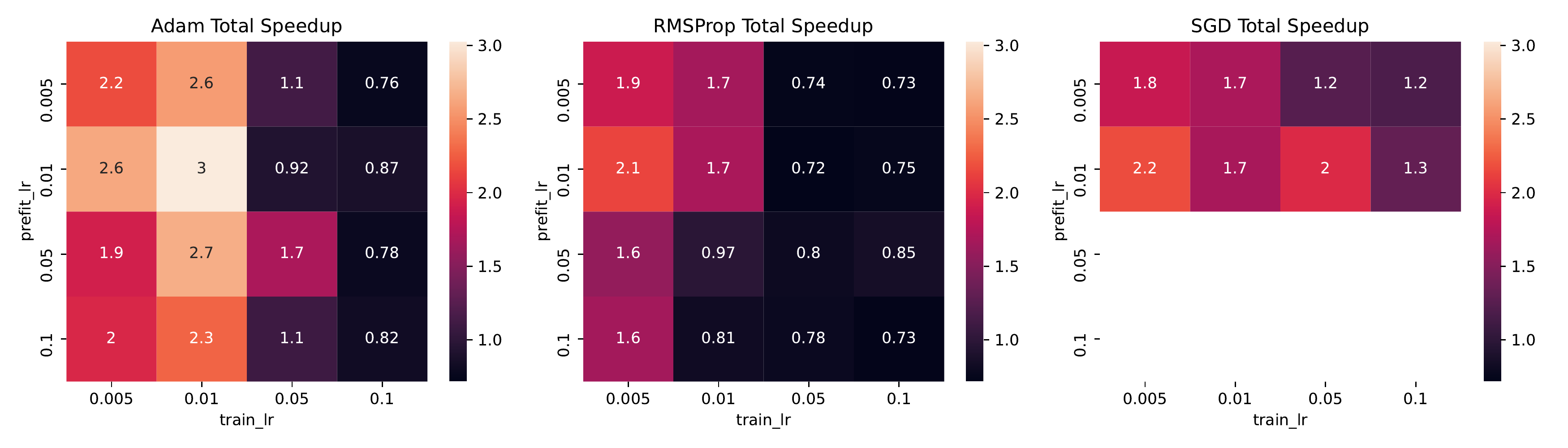}
    \caption{Total Speedup of different learning rates and optimizers.}
    \label{fig:Total_Speedup_Heatmap}
    
\end{figure}

\begin{figure}[htb]
    \vspace{0.5in}
    \centering
    \includegraphics[width=\linewidth]{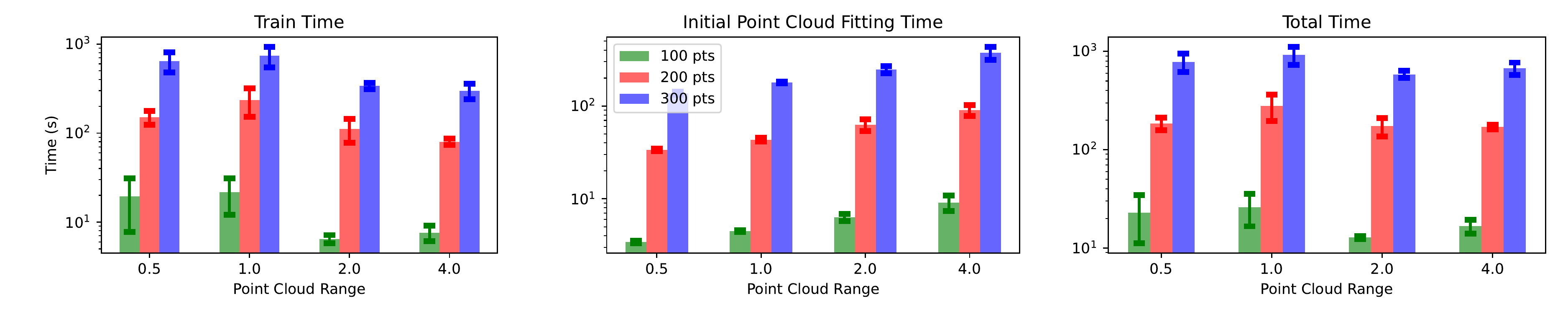}
    \caption{Training, prefitting and total time}
    \label{fig:GCN_runtime_regular}
    
    \vspace{0.25in}
    \centering
    \includegraphics[width=\linewidth]{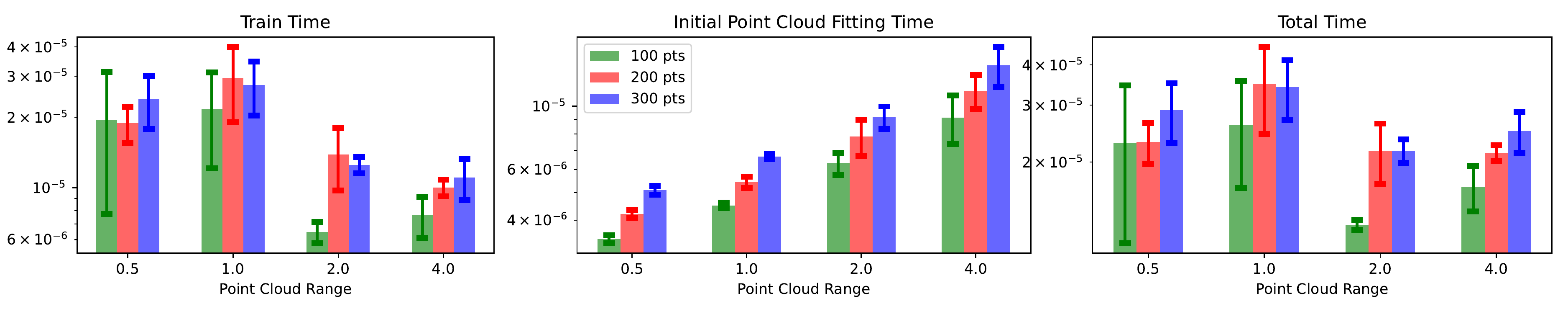}
    \caption{Density Normalized training, prefitting and total time. We normalize the runtime by $N^3$, where $N$ is the point cloud size. This is because persistence diagram computation has time complexity of $\mathcal{O}(N^3)$}.
    \label{fig:GCN_runtime_normalized}
\end{figure}

\paragraph{Detailed Runtime Comparison.} 
Fig. \ref{fig:GCN_runtime_regular} and \ref{fig:GCN_runtime_normalized} shows how training, initial point cloud fitting and total time evolve over different point cloud sizes and ranges. Training time decreases significantly with increasing range, especially from $1.0$ to $2.0$. This effect becomes more obvious with density normalized runtime. On the other hand, prefitting time increases exponentially with both point cloud range and size. Overall, the total time matches the trend of training time, however the speed-up is halved compared to training due to the addition of prefitting time.

\out{
\section{New experiments}

\subsection{Dynamic updating} 
Is it possible to efficiently update $M$ during optimization? 
For example, say we are doing classification on MNIST with SGD. 
The parameters of our model become $\vw$. 
SGD allows for a some stochasticity, which may allow us to approximate $\ba{M}$ from mini-batches. 
The main issue whether the computations are done efficiently, without significantly changing the time complexity. 
Computing $M$ requires finding the gradients $\ro \L/\ro \vw $. 
In the original problem, $\ro \L/\ro \vw$ needed to be calculated anyway for GD. 
Can we use the the $b$ vectors in the mini-batch as a rank $b$ approximation of $\ba{M}$? 
We could also add some memory, keeping a few of the past batches, but first consider only the most recent batch. 
For implementation, note that $\vw(\theta)$ are now intermediate states of the network. 
Therefore, direct gradients of them may not be accessible easily. 
what would be a good strategy for extracting and storing the gradients w.r.t. $\vw$? 
We need to test if we can use the grads of 
}